\documentclass[11pt,onecolumn]{article}

\usepackage[T1]{fontenc}
\usepackage{lmodern}
\usepackage{amsmath}
\usepackage{amssymb}
\usepackage{amsthm}
\usepackage{geometry}
\usepackage{booktabs}
\usepackage{array}
\usepackage{hyperref}
\usepackage{natbib}
\bibpunct{(}{)}{;}{a}{,}{,}

\newtheorem{theorem}{Theorem}
\newtheorem{lemma}{Lemma}
\newtheorem{proposition}{Proposition}
\newtheorem{corollary}{Corollary}[theorem]
\theoremstyle{remark}
\newtheorem{remark}{Remark}

\newcommand{\E}{\mathbb{E}}
\newcommand{\R}{\mathbb{R}}

\newcommand{\Prob}{\mathbb{P}}
\newcommand{\cP}{\mathcal{P}}
\newcommand{\cPh}{\widehat{\mathcal{P}}}
\newcommand{\cA}{\mathcal{A}}
\newcommand{\cS}{\mathcal{S}}
\newcommand{\cG}{\mathcal{G}}

\newcommand{\Var}{\mathrm{Var}}
\newcommand{\TV}{d_{\mathrm{TV}}}
\newcommand{\Wone}{W_1}
\DeclareMathOperator*{\argmin}{arg\,min}

\title{Reliability and Identifiability in Persona-Trained Monte Carlo:\\
Variance Decomposition, Stability Bounds, and the\\
Identifiability of Heterogeneous News Reaction}

\author{Salavat Ishbulatov \\
Independent researcher \\
\texttt{salavat@doplan.ai}}

\date{}

\begin{document}

\maketitle

\begin{abstract}
Persona-Trained Monte Carlo (PTMC) estimates distributions of market-outcome functionals by repeatedly simulating limit-order-book interaction among $K$ neural policy bots whose behavioral personas are drawn from a learned heterogeneity distribution $\cP$. This paper develops the statistical theory that makes the word ``reliable'' precise for such estimators.

We decompose estimator variance into a persona-draw component $\sigma_P^2$ and a within-run component $\sigma_w^2$, give unbiased ANOVA estimators of both, and derive the variance-optimal allocation of a fixed compute budget between outer persona draws and inner replications. A coupling-based stability bound quantifies how misestimation of $\cP$ and error in the trained policy propagate into the estimand, yielding a three-term total-error budget whose terms are separately estimable; a uniform-in-horizon version holds under a Doeblin condition on the market chain.

The main contribution is an identification theory for heterogeneous news reaction: under a fixed response nonlinearity, the aggregate impact curve $A(z)=\E_Q[g(\eta z)]$ detects heterogeneous news sensitivity through a strict Jensen gap and identifies the distribution $Q$ locally via odd moments and Hausdorff determinacy, with sharp failure when the response family is unknown. We provide $\sqrt{n}$-consistent estimators and a boundary-corrected test of homogeneous news reaction. Two separation theorems delimit when PTMC is provably preferable to homogeneous-population simulators and reduced-form forecasters, formalizing an irreducible Jensen bias floor and the Lucas critique as a minimax limit on intervention extrapolation. All proofs are given in full; guarantees are classified as unconditional (Monte Carlo convergence), conditional worst-case (the error budget), or open (the large-$K$ mean-field limit).
\end{abstract}

\noindent\textbf{Companion papers.} Framework specification: \emph{Persona-Trained Monte Carlo: Estimating Market-Outcome Distributions via Swarms of Persona-Conditioned Neural Policy Bots in a Limit Order Book} (\href{https://arxiv.org/abs/2606.29556}{arXiv:2606.29556}). This paper contains the theory; nothing here depends on simulation output.

\section{Introduction}
\label{sec:intro}

Agent-based market simulators promise distributions of outcomes---crash probabilities, drawdowns, spread dynamics---generated endogenously by interacting traders rather than imposed by a parametric price process. Persona-Trained Monte Carlo (PTMC) makes the promise statistical: each simulation run draws a population of $K$ trader personas from a distribution $\cP$, instantiates $K$ copies of one trained policy network $\pi_\phi$ conditioned on those personas, lets them trade in a continuous double auction for $T$ steps, and records a functional $F$ of the resulting path; averaging over $N$ independent runs gives a Monte Carlo estimator $\hat\mu_N$ of $\E_\cP[F]$. Three questions decide whether such an estimator deserves to be called reliable, and they are the subject of this paper.

\emph{First, what does the estimator's error look like, and how should a compute budget be spent?} Section~\ref{sec:estimator} decomposes the estimator's variance into a between-population component $\sigma_P^2$ (the variance of the conditional mean across persona draws) and a within-run component $\sigma_w^2$, exhibits unbiased ANOVA estimators of both from a two-stage design (Theorem~\ref{thm:consistency}), and derives the variance-minimizing allocation between outer persona draws and inner replications under a cost model (Theorem~\ref{thm:optimalR}). These results are mathematically standard---they are the input-uncertainty decomposition of the stochastic-simulation literature \citep{cheng_sensitivity_1997,barton_quantifying_2014,song_advanced_2014} and the nested-simulation allocation of \citet{sun_efficient_2011} transplanted to populations of interacting learned agents---and we present them as preliminaries with full proofs, because the quantity $\sigma_P^2$ is exactly what gives empirical content to the oldest objection to behavioral market models: \citet{gode_allocative_1993} showed that zero-intelligence traders reproduce allocative efficiency, so behavioral structure must earn its cost. Corollary~\ref{cor:gs} restates that objection as an allocation rule: if $\sigma_P^2=0$, the optimal design never redraws personas.

\emph{Second, how wrong is the answer when the inputs are wrong?} The persona distribution is estimated, and the policy is behaviorally cloned; both carry error into the estimand itself, beyond Monte Carlo noise. Section~\ref{sec:stability} proves a coupling bound (Theorem~\ref{thm:stability}): under a total-variation Lipschitz condition on the policy---verified for softmax policies with Lipschitz logits in Lemma~\ref{lem:softmax}---the estimand moves by at most $2MKTL\,\Wone(\cP,\cPh)$ when the persona distribution moves by $\Wone(\cP,\cPh)$ in Wasserstein distance, and by at most $2MKT\varepsilon$ when the policy is wrong by $\varepsilon$ in total variation. Corollary~\ref{cor:budget} assembles the resulting three-term reliability budget. The worst-case linearity in $KT$ parallels the compounding-error analyses of imitation learning \citep{ross_efficient_2010,ross_reduction_2011}; Theorem~\ref{thm:uniform} shows that under a Doeblin minorization of the market chain the horizon factor $T$ is replaced by $1/\alpha$, uniformly in $T$, for terminal-state and time-average functionals, by the perturbation theory of uniformly ergodic chains \citep{mitrophanov_sensitivity_2005,meyn_markov_2009}.

\emph{Third---and this is the paper's reason to exist---is the framework's central behavioral premise testable?} PTMC's premise is that each agent has a \emph{personal} reaction to external information: a per-agent news sensitivity $\eta^{(k)}$, drawn from a distribution $Q$, scaling a saturating response $g(\eta z)$ to a signed news shock of magnitude $z$. Section~\ref{sec:identification} proves that this premise is identifiable from aggregate data. If $\Var_Q(\eta)>0$, the aggregate impact curve $A(z)=\E_Q[g(\eta z)]$ lies strictly below the homogeneous curve with the same origin slope, at every $z>0$ (Theorem~\ref{thm:newsid}(i): a strict Jensen gap); knowledge of $A$ near zero identifies $Q$ up to nothing at all---the atom of news-ignorers included---via odd moments and Hausdorff determinacy (Theorem~\ref{thm:newsid}(ii)); if $A$ is known only up to an overall scale, as it is empirically, the \emph{shape} of the active part $Q(\cdot \mid \eta>0)$ remains identified, and heterogeneity of the active part remains detectable (Corollary~\ref{cor:scale}). Proposition~\ref{prop:unknowng} draws the sharp boundary of these claims: if the response function $g$ is unknown within the natural admissible class, identifiability fails completely---every heterogeneous population is observationally equivalent to a homogeneous population with a deformed response function. Identification of heterogeneous reaction to information is therefore a statement \emph{within} a model class whose response nonlinearity is fixed (as it is in PTMC, where $g$ is the architecture's output nonlinearity), not a model-free property of order flow. We are not aware of a statement of this identification problem, in either direction, in the agent-based-market literature; its mathematical home is the random-coefficients tradition of \citet{beran_estimating_1992}, \citet{ichimura_maximum_1998}, and \citet{gautier_nonparametric_2013}.

Section~\ref{sec:inference} completes the identification theory with inference: a nonlinear-least-squares estimator of $(\E\eta,\Var\eta)$ from event-study data is $\sqrt n$-consistent and asymptotically normal (Theorem~\ref{thm:nls}), and the natural test of the homogeneity null $H_0:\Var(\eta)=0$---a null on the boundary of the parameter space---has the $\tfrac12\chi^2_0+\tfrac12\chi^2_1$ limit of \citet{self_asymptotic_1987} and \citet{andrews_testing_2001} (Theorem~\ref{thm:boundary}). The boundary correction is not pedantic: the naive bootstrap is inconsistent in exactly this situation \citep{andrews_inconsistency_2000}, so a design that tested $\Var(\eta)=0$ by resampling would be wrong in a way that matters.

\emph{Fourth, when is PTMC provably the better tool?} Section~\ref{sec:separation} turns the comparison with rival method classes into theorems, in the only form such comparisons admit: conditional separations. Theorems~\ref{thm:quantjensen}--\ref{thm:homsep} give every homogeneous-population simulator an irreducible bias floor of explicit size $\tfrac38z^2\beta(z)\Var_Q(\eta)$ on impact functionals---an error no data or compute can remove, while the PTMC-class estimator is consistent. Theorem~\ref{thm:lucas} is the stronger statement: for an interventional target (a shift in the population's news sensitivity---a stress scenario), two markets are constructed that are \emph{observationally identical} yet differ in the target by an explicit $2\Delta$, so every estimator that is a functional of observational data---Bayesian neural networks, deep ensembles, and LLM-based forecasters included---has worst-case error at least $\Delta$ at every sample size, whereas the structural PTMC estimator is consistent when its response family is correct (Corollary~\ref{cor:lucasptmc}). This is the Lucas critique \citep{lucas_econometric_1976} converted into a minimax bound, with this paper's own observational-equivalence construction (Proposition~\ref{prop:unknowng}) as the two-point pair; Remark~\ref{rem:lucasreading} states with equal care what the theorem does \emph{not} claim---nothing about purely predictive targets, and nothing unconditional on the structural commitment.

\paragraph{What is and is not guaranteed.} Because ``convergence'' claims are easy to conflate, we state at the outset which one this paper proves, which it bounds, and which it leaves open.
\begin{itemize}
\item[(a)] \emph{Convergence in runs---proved, unconditionally.} For fixed population size $K$ and horizon $T$, $\hat\mu_N\to\E_{\cP}[F]$ almost surely at the canonical $N^{-1/2}$ rate with asymptotically valid confidence intervals (Theorem~\ref{thm:consistency}); the only hypotheses are boundedness of $F$ and independent runs, both true by construction. The limit, however, is the simulator's \emph{own} estimand.
\item[(b)] \emph{Distance to the correctly specified simulator---bounded, conditionally.} Corollary~\ref{cor:budget} controls the gap between that estimand and the one the same simulator would produce under the true persona distribution and the true policy, by two input-error terms---the Wasserstein error in the calibrated distribution of trader characteristics and the total-variation error of the cloned policy---each measurable from data. These terms do not shrink with more simulation, only with better inputs; the bound is worst-case, linear in $KT$, and informative precisely when the measured input errors are small on that scale (Remark~\ref{rem:budgetreading}).
\item[(c)] \emph{Two things no theorem here delivers.} The behavior of the estimand as $K\to\infty$ is stated as an open problem (Section~\ref{sec:discussion}), and no result asserts that the simulator class contains the real market: that is the empirical task of the companion validation protocol, for which Sections~\ref{sec:identification}--\ref{sec:inference} supply the sharpest available test.
\end{itemize}
What the theorems jointly deliver---and this is the sense in which they demonstrate the framework's potential---is a property rare among agent-based models: every component of PTMC's error either vanishes with compute or is attributable to a named, separately measurable input.

\paragraph{Honest positioning.} Theorems~\ref{thm:consistency} and~\ref{thm:optimalR} are instances of known results, included with proofs because the paper's later sections and its companion experiment design use their exact statements; the contribution there is the framing of persona draws as input uncertainty, not the mathematics. Theorem~\ref{thm:stability} and Lemma~\ref{lem:softmax} are, to our knowledge, new for interacting agent-based simulators, though the proof technique (one-step maximal coupling) is classical. Theorem~\ref{thm:uniform} is an application of \citet{mitrophanov_sensitivity_2005}-style perturbation bounds under an \emph{assumed} Doeblin condition whose verification for a realistic limit order book is open (Section~\ref{sec:discussion}). Theorems~\ref{thm:newsid}--\ref{thm:boundary}, Proposition~\ref{prop:unknowng}, and the separation results of Section~\ref{sec:separation} (Theorems~\ref{thm:quantjensen}--\ref{thm:lucas}) are the paper's contribution; the proof techniques of the latter are elementary (Taylor remainders, analytic continuation, a two-point argument), but we are not aware of their statement in the market-simulation context. No simulation results are reported; the numerical illustration the theory calls for is specified in Section~\ref{sec:numerics} and deferred to the companion experimental study.

\section{Related work}
\label{sec:related}

\paragraph{Input uncertainty in stochastic simulation.} When a simulation's driving distributions are themselves estimated, the estimator's error has a component that no amount of replication removes; quantifying it is the input-uncertainty problem \citep{cheng_sensitivity_1997,song_advanced_2014}. \citet{barton_quantifying_2014} construct confidence intervals accounting for input error via metamodel-assisted bootstrapping. PTMC's outer persona loop is input uncertainty of an unusual kind: the ``input distribution'' is a distribution over decision-making agents, and---unlike in the classical setting, where input uncertainty is a nuisance---the between-population variance $\sigma_P^2$ is itself an object of scientific interest, because it measures whether heterogeneity matters for the functional under study. The variance decomposition and ANOVA estimators of Theorem~\ref{thm:consistency} correspond to the classical one-way random-effects analysis \citep{searle_variance_1992}, and the optimal inner/outer allocation of Theorem~\ref{thm:optimalR} is the two-stage sampling trade-off known since \citet{cochran_sampling_1977}; in the nested-simulation form closest to ours, \citet{sun_efficient_2011} prove unbiasedness of the ANOVA-like variance estimator and derive optimal inner sample sizes under a budget.

\paragraph{Stability and perturbation.} The question ``how far does a simulator's output distribution move when its inputs move'' is treated for Markov chains by perturbation theory: for uniformly ergodic chains, one-step kernel perturbations translate into stationary and finite-horizon perturbations with constants governed by the convergence rate \citep{mitrophanov_sensitivity_2005,meyn_markov_2009}. Our Theorem~\ref{thm:stability} is not a chain-perturbation result---it couples two $K$-agent simulations agent by agent and step by step, paying a factor $KT$---but Theorem~\ref{thm:uniform} connects the two: under a Doeblin condition the chain-perturbation machinery replaces $T$ by a constant. The linear-in-$T$ worst case echoes the analysis of behavioral cloning under distribution shift \citep{ross_efficient_2010}, where per-step policy error $\varepsilon$ compounds to $O(T^2\varepsilon)$ in cumulative per-step loss and to $O(T\varepsilon)$ in the probability that the trajectory ever deviates; our bounded terminal functional corresponds to the latter regime. Optimal-transport distances enter only through the elementary coupling characterization of $\Wone$ \citep{villani_optimal_2009}.

\paragraph{Random-coefficient identification.} Recovering the distribution of a latent per-unit coefficient from aggregate or cross-sectional data is a classical program: \citet{beran_estimating_1992} for linear random-coefficient regression, \citet{ichimura_maximum_1998} and \citet{gautier_nonparametric_2013} for binary choice. Theorem~\ref{thm:newsid} belongs to this tradition with two twists specific to the market-simulation setting: the mixing function $g$ is a bounded, saturating, \emph{odd} nonlinearity fixed by the agent architecture, which makes the moment problem one on a compact support (Hausdorff, hence determinate \citep{shohat_problem_1943}); and the natural empirical curve is available only up to scale, which is why Corollary~\ref{cor:scale}'s up-to-scale statement, and the negative Proposition~\ref{prop:unknowng}, matter for practice. Boundary inference on a variance-type parameter goes back to \citet{chernoff_distribution_1954} and \citet{self_asymptotic_1987}, with the extremum-estimator generality we need in \citet{andrews_testing_2001} and the bootstrap caveat in \citet{andrews_inconsistency_2000}.

\paragraph{Agent-based market models.} The behavioral-relevance question that $\sigma_P^2$ quantifies originates with \citet{gode_allocative_1993}; the stylized facts that heterogeneity is supposed to buy are catalogued by \citet{cont_stylized_2001}; hand-coded heterogeneous-agent models that produce them include \citet{lux_volatility_1999}. Mean-field and scaling limits of order-book models \citep{lachapelle_efficiency_2016,horst_law_2017} are the natural home for the open $K\to\infty$ questions we state but do not solve. The behavioral persona components whose distributions $\cP$ encodes are grounded in \citet{kahneman_prospect_1979,tversky_loss_1991,odean_overconfidence_1998}; the news-novelty channel that motivates the identification target of Section~\ref{sec:identification} is documented by \citet{glasserman_new_2023}.

\section{Formal setup}
\label{sec:setup}

\subsection{The simulator as a Markov chain}
\label{sec:simulator}

Throughout, the following objects are fixed. A \emph{persona space} $P\subset\R^{d_p}$, compact, with norm $\|\cdot\|$; personas are tuples $p=(\theta,\rho)$ of demographic and behavioral coordinates as in the companion framework paper. A \emph{persona distribution} $\cP$, a Borel probability measure on $P$. A \emph{state space} $\cS$, a Polish space carrying the full market configuration: the limit order book, price and news history, and every bot's inventory and wealth. A finite \emph{action space} $\cA$ (order type and discretized size/price; Remark~\ref{rem:continuous} treats continuous heads). A \emph{trained policy} $\pi_\phi$: a Markov kernel from $\cS\times P$ to $\cA$, so that bot $k$ in state $s$ with persona $p^{(k)}$ plays $a\sim\pi_\phi(\cdot\mid s,p^{(k)})$. An i.i.d.\ \emph{exogenous shock} sequence $(\xi_t)_{t\ge1}$ with values in a Polish space $\Xi$ and law $\nu_\xi$, independent of everything else (news arrivals, liquidity shocks); the degenerate case $\xi_t\equiv\text{const}$ is allowed. A measurable, \emph{deterministic matching map} $\Psi:\cS\times\cA^K\times\Xi\to\cS$: given the state, the action profile, and the shock, the continuous double auction clears deterministically.

A \emph{run} is generated as follows: draw a persona population $\omega=(p^{(1)},\dots,p^{(K)})\sim\cP^{\otimes K}$; fix a (deterministic) initial state $s_0$; for $t=0,\dots,T-1$, draw actions $a_t^{(k)}\sim\pi_\phi(\cdot\mid s_t,p^{(k)})$ independently across $k$ given $s_t$, and set $s_{t+1}=\Psi(s_t,(a_t^{(1)},\dots,a_t^{(K)}),\xi_{t+1})$. The \emph{path} is $X=(s_0,\dots,s_T)\in\cS^{T+1}$. A \emph{functional} $F:\cS^{T+1}\to\R$ is Borel measurable with $|F|\le M$; all functionals in the companion design (maximum drawdown on a capped price grid, crash indicators, capped tail statistics, spread summaries) are of this form by construction. Write
\[
\mu \;=\; \E[F(X)], \qquad \mu(\omega) \;=\; \E[F(X)\mid\omega],
\]
where $\mu(\omega)$ is a measurable function of $\omega$ (Fubini's theorem, all kernels being Markov kernels between Polish spaces and $F$ bounded). Conditional on $\omega$, the state sequence $(s_t)_{t\ge0}$ is a time-homogeneous Markov chain on $\cS$ with transition kernel
\begin{equation}
P_\omega(s,B) \;=\; \int_\Xi \sum_{a\in\cA^K} \Big[\prod_{k=1}^K \pi_\phi(a^{(k)}\mid s,p^{(k)})\Big]\, \mathbf 1\{\Psi(s,a,\xi)\in B\}\, \nu_\xi(d\xi).
\label{eq:kernel}
\end{equation}

Three assumptions are used at various points; each theorem states which it needs.
\begin{itemize}
\item[\textbf{(A1)}] \emph{(TV-Lipschitz policy.)} There is $L<\infty$ with
$\sup_{s\in\cS}\TV\big(\pi_\phi(\cdot\mid s,p),\,\pi_\phi(\cdot\mid s,p')\big)\le L\|p-p'\|$ for all $p,p'\in P$,
where $\TV(\lambda,\lambda')=\sup_{B}|\lambda(B)-\lambda'(B)|$ denotes total variation distance. Lemma~\ref{lem:softmax} verifies (A1) for softmax policies with Lipschitz logits.
\item[\textbf{(A2)}] \emph{(Bounded functional.)} $|F|\le M$.
\item[\textbf{(A3)}] \emph{(Deterministic matching, common shocks.)} The dynamics are as above: matching is a deterministic measurable function of (state, action profile, shock), and when two simulations are compared, the exogenous shocks have the same law in both (and can hence be coupled to coincide).
\end{itemize}

\subsection{The PTMC estimator and its two-stage refinement}

The single-stage estimator is $\hat\mu_N=\frac1N\sum_{i=1}^N F(X_i)$ over $N$ i.i.d.\ runs (independent persona populations \emph{and} independent within-run randomness). The two-stage design draws $D$ persona populations $\omega_1,\dots,\omega_D\sim\cP^{\otimes K}$ i.i.d.; for each $i$ it runs $R$ conditionally independent replications sharing $\omega_i$, yielding $F_{ij}$, $j=1,\dots,R$, and sets $\bar F_i=\frac1R\sum_j F_{ij}$, $\hat\mu_{D,R}=\frac1D\sum_i\bar F_i$. The estimator randomness thus has exactly the two sources the theory separates: which population was drawn, and what happened within the run.

\section{Estimator theory: variance decomposition and budget allocation}
\label{sec:estimator}

Define the \emph{between-population} and \emph{within-run} variance components
\[
\sigma_P^2 \;=\; \Var\big(\mu(\omega)\big), \qquad
\sigma_w^2 \;=\; \E\big[\Var\big(F(X)\mid\omega\big)\big],
\]
both finite by (A2), and note $\Var(F(X))=\sigma_P^2+\sigma_w^2$ by the law of total variance.

\begin{theorem}[Consistency, CLT, and variance decomposition]
\label{thm:consistency}
Assume (A2), runs i.i.d.\ as above. Then:
\begin{enumerate}
\item $\E[\hat\mu_N]=\mu$ and $\hat\mu_N\to\mu$ almost surely as $N\to\infty$.
\item If $\sigma_P^2+\sigma_w^2>0$, then $\sqrt N(\hat\mu_N-\mu)\Rightarrow\mathcal N(0,\sigma_P^2+\sigma_w^2)$.
\item In the two-stage design, $\E[\hat\mu_{D,R}]=\mu$,
\[
\Var(\hat\mu_{D,R}) \;=\; \frac{\sigma_P^2}{D}+\frac{\sigma_w^2}{DR},
\]
and, for $R\ge2$, $D \ge 2$, the one-way random-effects ANOVA estimators
\[
\hat\sigma_w^2=\frac{1}{D(R-1)}\sum_{i=1}^D\sum_{j=1}^R\big(F_{ij}-\bar F_i\big)^2,
\qquad
\hat\sigma_P^2=\frac{1}{D-1}\sum_{i=1}^D\big(\bar F_i-\hat\mu_{D,R}\big)^2-\frac{\hat\sigma_w^2}{R}
\]
satisfy $\E[\hat\sigma_w^2]=\sigma_w^2$ and $\E[\hat\sigma_P^2]=\sigma_P^2$. Moreover, if $\sigma_P^2+\sigma_w^2/R>0$, then $\sqrt D(\hat\mu_{D,R}-\mu)\Rightarrow\mathcal N(0,\sigma_P^2+\sigma_w^2/R)$ as $D\to\infty$ with $R$ fixed.
\end{enumerate}
\end{theorem}

\begin{proof}
(1)--(2). The variables $F_i=F(X_i)$ are i.i.d.\ with mean $\mu$ and variance $\Var(F_i)=\Var(\E[F_i\mid\omega_i])+\E[\Var(F_i\mid\omega_i)]=\sigma_P^2+\sigma_w^2<\infty$. Unbiasedness is linearity of expectation; almost-sure convergence is Kolmogorov's strong law; the central limit theorem is Lindeberg--L\'evy.

(3). Conditionally on $\omega_i$, the $F_{ij}$, $j=1,\dots,R$, are i.i.d.\ with mean $\mu(\omega_i)$ and variance $v(\omega_i):=\Var(F\mid\omega_i)$, and $\E[v(\omega_i)]=\sigma_w^2$. Hence $\E[\bar F_i\mid\omega_i]=\mu(\omega_i)$, $\Var(\bar F_i\mid\omega_i)=v(\omega_i)/R$, and by the law of total variance
\[
\Var(\bar F_i)=\Var(\mu(\omega_i))+\E\!\left[\frac{v(\omega_i)}{R}\right]=\sigma_P^2+\frac{\sigma_w^2}{R}.
\]
The rows $(\bar F_i)_{i\le D}$ are i.i.d., so $\Var(\hat\mu_{D,R})=(\sigma_P^2+\sigma_w^2/R)/D$, which is the displayed formula, and the stated CLT for $\hat\mu_{D,R}$ is again Lindeberg--L\'evy applied to the i.i.d.\ bounded variables $\bar F_i$.

For $\hat\sigma_w^2$: conditionally on $\omega_i$, $\sum_j(F_{ij}-\bar F_i)^2$ is $(R-1)$ times the usual unbiased sample variance of an i.i.d.\ sample, so $\E\big[\sum_j(F_{ij}-\bar F_i)^2\mid\omega_i\big]=(R-1)v(\omega_i)$; taking expectations and summing over $i$ gives $\E[\hat\sigma_w^2]=\sigma_w^2$.

For $\hat\sigma_P^2$: the $\bar F_i$ are i.i.d.\ with variance $\sigma_P^2+\sigma_w^2/R$, so the sample variance $\frac{1}{D-1}\sum_i(\bar F_i-\hat\mu_{D,R})^2$ has expectation $\sigma_P^2+\sigma_w^2/R$; subtracting $\hat\sigma_w^2/R$, which has expectation $\sigma_w^2/R$, leaves $\E[\hat\sigma_P^2]=\sigma_P^2$.
\end{proof}

\medskip

\begin{remark}[Using $\hat\sigma_P^2$ as a test statistic]
\label{rem:permutation}
$\hat\sigma_P^2$ can be negative in samples (a familiar feature of ANOVA variance-component estimation \citep{searle_variance_1992}); for point reporting it is customary to truncate at zero, at the price of a small positive bias. As a \emph{test}, the natural statistic is $\mathrm{MS}_{\mathrm{between}}/\mathrm{MS}_{\mathrm{within}}$ with $\mathrm{MS}_{\mathrm{between}}=\frac{R}{D-1}\sum_i(\bar F_i-\hat\mu)^2$ and $\mathrm{MS}_{\mathrm{within}}=\hat\sigma_w^2$. Two nulls must be distinguished. Under the \emph{sharp null} that the conditional law of $F$ given $\omega$ does not depend on $\omega$, all $DR$ observations are i.i.d., group labels are exchangeable, and the permutation test (re-randomizing runs across persona groups) is exact at any sample size \citep{lehmann_testing_2005}. Under the weaker \emph{moment null} $\sigma_P^2=0$, the conditional variance $v(\omega)$ may still vary with $\omega$; exchangeability then fails and the permutation test is only approximately valid. The classical $F_{D-1,D(R-1)}$ reference distribution is exact under the sharp null \emph{plus} conditional normality with constant $v(\omega)$. The companion experiment design pre-registers the permutation version and interprets rejections as evidence against the sharp null, which is the scientifically relevant one: it is exactly the statement ``persona draws do not affect the outcome distribution.''
\end{remark}

\begin{theorem}[Optimal inner/outer allocation]
\label{thm:optimalR}
Let $c_0>0$ be the cost of one outer draw (persona sampling and market setup) and $c_1>0$ the cost of one inner replication, so a $(D,R)$ design costs $B=D(c_0+Rc_1)$. Treat $D,R$ as continuous with $\sigma_P^2 > 0$, $\sigma^2_w > 0$. Then, subject to the budget $B$, $\Var(\hat\mu_{D,R})=(\sigma_P^2+\sigma_w^2/R)/D$ is minimized at
\[
R^*=\sqrt{\frac{c_0\,\sigma_w^2}{c_1\,\sigma_P^2}},
\qquad\text{with minimal variance}\qquad
V^*=\frac{\big(\sqrt{c_0\sigma_P^2}+\sqrt{c_1\sigma_w^2}\big)^2}{B}.
\]
\end{theorem}

\begin{proof}
Substituting $D=B/(c_0+Rc_1)$,
\[
B\cdot V(R)\;=\;\Big(\sigma_P^2+\frac{\sigma_w^2}{R}\Big)(c_0+Rc_1)
\;=\;\sigma_P^2c_0+\sigma_w^2c_1+\Big(\sigma_P^2c_1R+\frac{\sigma_w^2c_0}{R}\Big).
\]
The bracketed term is strictly convex in $R>0$ and, by the AM--GM inequality, is at least $2\sqrt{\sigma_P^2\sigma_w^2c_0c_1}$, with equality iff $\sigma_P^2c_1R=\sigma_w^2c_0/R$, i.e.\ iff $R=R^*$. At $R^*$,
$B\,V(R^*)=\sigma_P^2c_0+\sigma_w^2c_1+2\sqrt{\sigma_P^2c_0\,\sigma_w^2c_1}=\big(\sqrt{c_0\sigma_P^2}+\sqrt{c_1\sigma_w^2}\big)^2$.
\end{proof}

\medskip

\begin{remark}
In practice $R$ is an integer; since $B\,V(R)$ is convex in $R$, the integer optimum is one of $\lfloor R^*\rfloor,\lceil R^*\rceil$. The pilot-then-allocate use of Theorem~\ref{thm:optimalR}---estimate $(\hat\sigma_P^2,\hat\sigma_w^2)$ from a small two-stage pilot via Theorem~\ref{thm:consistency}(3), then set $R$---is the standard two-stage sampling design \citep{cochran_sampling_1977}; in the nested-simulation setting, \citet{sun_efficient_2011} derive the same square-root allocation and show the optimal inner sample size stays bounded as $B\to\infty$.
\end{remark}

\begin{corollary}[The Gode--Sunder objection as an allocation rule]
\label{cor:gs}
Fix $c_0,c_1,\sigma_w^2>0$. As $\sigma_P^2\to0$, $R^*\to\infty$: in the limit of irrelevant personas the optimal design draws one population and never redraws it. Conversely $R^*$ is small precisely when $\sigma_P^2$ is large relative to $\sigma_w^2$. Thus the measured $R^*$ is itself a scale-free summary of how much learned heterogeneity matters for the functional under study, and estimating it answers, functional by functional, the objection of \citet{gode_allocative_1993} that behavioral structure may be doing no work.
\end{corollary}

\section{Stability: reliability under misestimated inputs}
\label{sec:stability}

Monte Carlo error is only one of three error sources: the persona distribution $\cPh$ is estimated, and the policy $\hat\pi$ is trained. This section bounds the movement of the estimand itself.

\subsection{Verifying assumption (A1) in practice}

Assumption (A1) is not a special or fragile requirement: it holds for the standard policy architectures used in practice, with an explicit, computable constant $L$ in every case. For a finite action space with a softmax head and Lipschitz logits, (A1) holds with $L$ no larger than the logits' own Lipschitz constant; for composite actions with a continuous (e.g.\ Gaussian) mark head, an analogous bound holds with an extra additive term controlled by the mark head's Lipschitz behavior; and even if the mark head's own scale varies with the persona, an explicit bound remains available via Pinsker's inequality. The precise statement and proof---Lemma~\ref{lem:softmax} and Remark~\ref{rem:scale}---are given in Appendix~\ref{app:softmax}, since they are a technical verification for a specific architecture rather than part of the paper's argument; what the results of this section use is only that (A1) holds with some finite $L$, verified once and for all in the appendix for the architecture the companion framework actually trains.

\subsection{The finite-horizon stability bound}

$\Wone$ denotes the Wasserstein-1 distance on probability measures on $P$ with cost $\|\cdot\|$: $\Wone(\cP,\cPh)=\inf_\gamma\E_{(p,\hat p)\sim\gamma}\|p-\hat p\|$ over couplings $\gamma$ of $(\cP,\cPh)$; an optimal coupling exists \citep[Thm.~4.1]{villani_optimal_2009}.

\begin{theorem}[Stability under misestimation of $\cP$ and of $\pi$]
\label{thm:stability}
Assume (A1)--(A3), with both simulations started at the same $s_0$ and run for $T$ steps with $K$ bots.
\begin{enumerate}
\item For any two persona distributions $\cP,\cPh$ on $P$,
\[
\Big|\E_{\cP^{\otimes K}}[F(X)]-\E_{\cPh^{\otimes K}}[F(X)]\Big|\;\le\;2MKTL\;\Wone(\cP,\cPh).
\]
\item If a second policy $\hat\pi$ satisfies $\sup_{s,p}\TV\big(\hat\pi(\cdot\mid s,p),\pi_\phi(\cdot\mid s,p)\big)\le\varepsilon$, then, with the same persona distribution in both simulations,
\[
\Big|\E^{\hat\pi}[F(X)]-\E^{\pi_\phi}[F(X)]\Big|\;\le\;2MKT\varepsilon.
\]
\end{enumerate}
\end{theorem}

\noindent Where $M$ bounds $F$ via (A2), $L$ is the policy's TV-Lipschitz constant via (A1), and $K,T$ are the population size and horizon fixed in Section~\ref{sec:simulator}.

\begin{proof}
\emph{Step 1: fixed populations.} Fix $\omega=(p^{(k)})_{k\le K}$ and $\hat\omega=(\hat p^{(k)})_{k\le K}$ and set
$\delta(\omega,\hat\omega)=L\sum_{k=1}^K\|p^{(k)}-\hat p^{(k)}\|$.
We construct the two runs on a common probability space. Both use the same shock sequence $(\xi_t)$ (possible by (A3)). Suppose the histories agree up to time $t$, with common state $s_t$. Conditionally, the two action profiles are product measures over $k$; couple them coordinatewise, each bot $k$ by a maximal coupling of $\pi_\phi(\cdot\mid s_t,p^{(k)})$ and $\pi_\phi(\cdot\mid s_t,\hat p^{(k)})$, drawn independently across $k$. (On the finite space $\cA$ the maximal coupling is explicit and measurable in its two input pmfs: with probability $1-\tau$, where $\tau$ is the total variation distance, draw both actions equal from the normalized overlap $\min(q,\hat q)/(1-\tau)$; otherwise draw the pair independently from the normalized positive and negative parts of $q - \hat q$.) By (A1) and the union bound,
\[
\Prob\big(\exists k:\;a_t^{(k)}\ne\hat a_t^{(k)}\;\big|\;\text{histories agree at }t\big)
\;\le\;\sum_{k=1}^K L\|p^{(k)}-\hat p^{(k)}\|\;=\;\delta(\omega,\hat\omega).
\]
If all $K$ coupled action pairs agree, then by (A3) the next states coincide, $s_{t+1}=\Psi(s_t,a_t,\xi_{t+1})=\hat s_{t+1}$, and the construction continues; after the first disagreement, let the two runs evolve independently according to their own laws. Writing $E_T$ for the event that the full paths agree,
\[
\Prob(E_T^c)\;\le\;\sum_{t=0}^{T-1}\Prob\big(\text{first disagreement at }t\big)\;\le\;T\,\delta(\omega,\hat\omega),
\]
and by (A2), since $F(X)=F(\hat X)$ on $E_T$,
\begin{equation}
\big|\mu(\omega)-\mu(\hat\omega)\big|=\big|\E[F(X)-F(\hat X)]\big|\le2M\,\Prob(E_T^c)\le2MT\,\delta(\omega,\hat\omega).
\label{eq:fixedpop}
\end{equation}

\emph{Step 2: random populations.} Let $\gamma$ be a $\Wone$-optimal coupling of $(\cP,\cPh)$ and draw the pairs $(p^{(k)},\hat p^{(k)})$, $k=1,\dots,K$, i.i.d.\ from $\gamma$; then $\omega\sim\cP^{\otimes K}$, $\hat\omega\sim\cPh^{\otimes K}$, and $\E[\delta(\omega,\hat\omega)]=LK\,\Wone(\cP,\cPh)$. By \eqref{eq:fixedpop} and Jensen's inequality,
\[
\big|\E_{\cP^{\otimes K}}F-\E_{\cPh^{\otimes K}}F\big|
=\big|\E_\gamma[\mu(\omega)-\mu(\hat\omega)]\big|
\le2MT\,\E[\delta(\omega,\hat\omega)]=2MKTL\,\Wone(\cP,\cPh).
\]

(2) is the same coupling with the roles of the two \emph{policies} exchanged for the two persona vectors: with a common population $\omega$, couple bot $k$'s action at each step by a maximal coupling of $\hat\pi(\cdot\mid s_t,p^{(k)})$ and $\pi_\phi(\cdot\mid s_t,p^{(k)})$, which disagrees with probability at most $\varepsilon$; the per-step disagreement probability is at most $K\varepsilon$, and \eqref{eq:fixedpop} becomes $|\Delta\E F|\le2MTK\varepsilon$.
\end{proof}

\medskip

\begin{corollary}[Total reliability budget]
\label{cor:budget}
Let $\mu^{\mathrm{true}}$ be the estimand under the true persona distribution $\cP$ and the true policy $\pi^*$, and let $\hat\mu_N$ be the PTMC estimator computed from $N$ i.i.d.\ runs under the estimated pair $(\cPh,\hat\pi)$, with $\mu^{\mathrm{sim}}$ its estimand. Assume (A2)--(A3), (A1) for $\hat{\pi}$, $\sup_{s,p}\TV(\hat\pi,\pi^*)\le\varepsilon_{\mathrm{BC}}$, and $\sigma_P^2+\sigma_w^2>0$ under $(\cPh,\hat\pi)$. Then, for any $\alpha\in(0,1)$,
\[
\limsup_{N\to\infty}\;\Prob\Bigg(\big|\hat\mu_N-\mu^{\mathrm{true}}\big|\;>\;
z_{1-\alpha/2}\,\frac{\sqrt{\sigma_P^2+\sigma_w^2}}{\sqrt N}
\;+\;2MKT\,\varepsilon_{\mathrm{BC}}
\;+\;2MKTL\,\Wone(\cP,\cPh)\Bigg)\;\le\;\alpha .
\]
\end{corollary}

\begin{proof}
Decompose $\hat\mu_N-\mu^{\mathrm{true}}=(\hat\mu_N-\mu^{\mathrm{sim}})+(\mu^{\mathrm{sim}}-\mu^{\mathrm{mid}})+(\mu^{\mathrm{mid}}-\mu^{\mathrm{true}})$, where $\mu^{\mathrm{mid}}$ is the estimand under $(\cP,\hat\pi)$. The middle difference is controlled by Theorem~\ref{thm:stability}(1) applied to the policy $\hat\pi$ (which satisfies (A1)): $|\mu^{\mathrm{sim}}-\mu^{\mathrm{mid}}|\le2MKTL\,\Wone(\cPh,\cP)$. The last difference is Theorem~\ref{thm:stability}(2) applied under $\cP^{\otimes K}$: $|\mu^{\mathrm{mid}}-\mu^{\mathrm{true}}|\le2MKT\varepsilon_{\mathrm{BC}}$. The first difference satisfies the CLT of Theorem~\ref{thm:consistency}(2), giving the asymptotic probability bound.
\end{proof}

\medskip

\begin{remark}[Reading the budget]
\label{rem:budgetreading}
Three notes on scope and use. (i) \emph{What ``true'' means.} $\mu^{\mathrm{true}}$ is the estimand of the correctly specified \emph{simulator}---the same matching engine, horizon, and population size, run under the true persona distribution and the true policy. The corollary attributes input error; it does not certify the simulator class against reality, which is the task of the companion validation protocol. (ii) \emph{Estimability.} The three terms are separately estimable: the Monte Carlo term from Theorem~\ref{thm:consistency}(3) with a plug-in variance; $\varepsilon_{\mathrm{BC}}$ from held-out behavioral-cloning evaluation (per-state total variation between the cloned policy and held-out action frequencies, or a calibrated upper bound via log-loss and Pinsker); $\Wone(\cP,\cPh)$ from the persona-calibration bootstrap---the latter is exactly the uncertainty in the distribution of trader characteristics that the calibration data leave unresolved. (iii) \emph{Worst case, and when it binds.} Since $|\Delta\E F|\le2M$ holds trivially, the policy and persona terms are informative only when $KT\varepsilon_{\mathrm{BC}}<1$ and $KTL\,\Wone(\cP,\cPh)<1$; with realistic $K$ and $T$ this is demanding, because the constants price the worst case in which every bot's every decision is pivotal. This is why the companion experiment design (its Phase~1, step~5) measures the \emph{realized} sensitivity slopes $|\Delta\hat\mu|/\Wone$ and $|\Delta\hat\mu|/\varepsilon_{\mathrm{BC}}$, expected to lie orders of magnitude below the worst-case constants: the theorem's value is the decomposition of the error into named, separately attributable channels, not the tightness of its constants. The linear-in-$T$ growth for a bounded terminal functional matches the trajectory-deviation regime of imitation-learning theory \citep{ross_efficient_2010,ross_reduction_2011}; the $O(T^2)$ compounding familiar from that literature concerns cumulative per-step loss, a different functional class.
\end{remark}

\begin{remark}[Continuous action heads]
\label{rem:continuous}
Theorem~\ref{thm:stability} is stated for finite $\cA$ only through the explicit maximal-coupling construction; Lemma~\ref{lem:softmax}(2)--(3) supplies (A1) for composite actions with continuous marks, and the coupling argument goes through verbatim on standard Borel action spaces, where maximal couplings measurable in their inputs exist by the usual density construction (write both kernels' densities with respect to a common dominating measure, e.g.\ counting $\otimes$ Lebesgue for softmax-plus-Gaussian heads, and repeat the overlap construction).
\end{remark}

\subsection{A uniform-in-horizon bound under a Doeblin condition}

Theorem~\ref{thm:stability} degrades linearly in $T$. For functionals of the \emph{terminal} state, or time averages, the horizon factor can be removed if the market chain forgets its past at a uniform geometric rate. We state the ergodicity hypothesis as an assumption; whether realistic order-book chains satisfy it is discussed honestly below.

\begin{itemize}
\item[\textbf{(A4)}] \emph{(Uniform Doeblin minorization.)} There exists $\alpha\in(0,1]$ such that for every persona population $\omega\in P^K$ there is a probability measure $\nu_\omega$ on $\cS$ with
$P_\omega(s,\cdot)\;\ge\;\alpha\,\nu_\omega(\cdot)$ for all $s\in\cS$,
where $P_\omega$ is the one-step kernel \eqref{eq:kernel}.
\end{itemize}

\begin{theorem}[Uniform-in-$T$ stability]
\label{thm:uniform}
Assume (A1), (A3), (A4). Let $F(X)=f(s_T)$ with $|f|\le M$, or $F(X)=\frac1T\sum_{t=1}^Tf_t(s_t)$ with $\sup_t|f_t|\le M$. Then for all $T\ge1$,
\[
\Big|\E_{\cP^{\otimes K}}[F(X)]-\E_{\cPh^{\otimes K}}[F(X)]\Big|\;\le\;\frac{2MKL}{\alpha}\,\Wone(\cP,\cPh),
\]
and likewise $|\Delta\E F|\le2MK\varepsilon/\alpha$ for a policy perturbation as in Theorem~\ref{thm:stability}(2). More precisely, the factor $1/\alpha$ can be replaced by $\min(T,1/\alpha)$.
\end{theorem}

\begin{proof}
\emph{Step 0: contraction.} Under (A4), write $P_\omega=\alpha\nu_\omega+(1-\alpha)R_\omega$ with $R_\omega$ a Markov kernel. For probability measures $\lambda,\lambda'$, $(\lambda-\lambda')P_\omega=(1-\alpha)(\lambda-\lambda')R_\omega$ because the $\nu_\omega$ parts cancel, whence the Dobrushin-type contraction
\begin{equation}
\TV(\lambda P_\omega,\lambda'P_\omega)\le(1-\alpha)\,\TV(\lambda,\lambda').
\label{eq:dobrushin}
\end{equation}

\emph{Step 1: one-step kernel perturbation.} Fix $\omega,\hat\omega$ and let $\Delta:=\sup_{s\in\cS}\TV\big(P_\omega(s,\cdot),P_{\hat\omega}(s,\cdot)\big)$. By \eqref{eq:kernel}, $P_\omega(s,\cdot)$ is the pushforward under the measurable map $(a,\xi)\mapsto\Psi(s,a,\xi)$ of the product law (action profile under $\omega$) $\otimes\,\nu_\xi$. Total variation does not increase under a common measurable pushforward, and for product measures with a common factor it is bounded by the total variation of the differing factors, which for products over $k$ is at most the sum of the coordinatewise distances. Hence, using (A1),
\begin{equation}
\Delta\;\le\;\sup_s\;\TV\Big(\textstyle\bigotimes_k\pi_\phi(\cdot\mid s,p^{(k)}),\;\bigotimes_k\pi_\phi(\cdot\mid s,\hat p^{(k)})\Big)
\;\le\;L\sum_{k=1}^K\|p^{(k)}-\hat p^{(k)}\|\;=\;\delta(\omega,\hat\omega).
\label{eq:onestep}
\end{equation}
(The subadditivity over coordinates follows from Lemma~\ref{lem:softmax}(2)'s decomposition argument applied $K-1$ times, or directly from the coordinatewise maximal coupling and a union bound.)

\emph{Step 2: telescoping.} Let $\lambda_t=\delta_{s_0}P_\omega^t$ and $\hat\lambda_t=\delta_{s_0}P_{\hat\omega}^t$ be the time-$t$ state laws. The telescoping identity
\[
\hat\lambda_T-\lambda_T\;=\;\sum_{t=0}^{T-1}\hat\lambda_t\big(P_{\hat\omega}-P_\omega\big)P_\omega^{\,T-1-t}
\]
holds by expanding the difference of products. Each summand is a signed measure of total mass zero; its total variation after the first factor is at most $\Delta$ (integrate the state-wise bound over $\hat\lambda_t$), and each subsequent application of $P_\omega$ contracts by $(1-\alpha)$ by \eqref{eq:dobrushin}. Hence
\[
\TV(\hat\lambda_T,\lambda_T)\;\le\;\Delta\sum_{t=0}^{T-1}(1-\alpha)^{T-1-t}\;\le\;\Delta\,\min\!\Big(T,\frac1\alpha\Big).
\]
For $F=f(s_T)$, $|\E f(s_T)-\hat\E f(s_T)|\le2M\,\TV(\hat\lambda_T,\lambda_T)\le2M\Delta\min(T,1/\alpha)$; for the time average, apply the bound at each $t\le T$ and average, which only improves it. Combining with \eqref{eq:onestep} and integrating over the $\Wone$-optimal coupling as in Step~2 of Theorem~\ref{thm:stability} gives the claim; the policy version replaces \eqref{eq:onestep} by $\Delta\le K\varepsilon$.
\end{proof}

\medskip

\begin{remark}[Scope and honesty]
\label{rem:doeblin}
Three caveats. (i) Path-supremum functionals (maximum drawdown, crash-within-window indicators over the whole path) are not covered: the argument controls marginals, not path laws, and a uniform-in-$T$ bound cannot hold for, e.g., running maxima of an ergodic chain without further structure. For such functionals Theorem~\ref{thm:stability}'s linear-in-$T$ bound stands. (ii) Verifying (A4) for a price--time-priority order book on its natural (unbounded) state space is an open problem; we do not claim it. On a truncated state space (capped prices, book depth, and inventories---which the companion design imposes anyway to keep $F$ bounded), softmax policies have full support on $\cA$, so every action profile has positive probability and minorization arguments are available in principle, but the resulting $\alpha$ may be astronomically small, in which case $\min(T,1/\alpha)=T$ for all practical horizons and Theorem~\ref{thm:uniform} adds nothing beyond Theorem~\ref{thm:stability}. The theorem's value is structural: it identifies the single constant ($\alpha$, equivalently the chain's mixing rate) that decides whether reliability degrades with horizon. Sharper constants in terms of ergodicity coefficients of iterated kernels are in \citet{mitrophanov_sensitivity_2005}; background on minorization and uniform ergodicity is in \citet{meyn_markov_2009}. (iii) (A4) with $\nu_\omega$ allowed to depend on $\omega$ is weaker than a joint minorization; only the contraction \eqref{eq:dobrushin} for each fixed $\omega$ is used.
\end{remark}

\section{Identifiability of heterogeneous news reaction}
\label{sec:identification}

This section is the paper's core. The behavioral premise under examination is that each agent has a \emph{personal} reaction to external information: agent $k$'s expected instantaneous net order-flow response to a signed news shock of magnitude $z$ is $g(\eta^{(k)}z)$, where $g$ is a saturating nonlinearity fixed by the agent architecture and $\eta^{(k)}\ge0$ is the agent's news sensitivity, drawn from a distribution $Q$---the marginal of $\cP$ on the news-sensitivity coordinate. Aggregating over the population gives the \emph{aggregate news-impact curve}
\[
A(z)\;=\;\E_{Q}\big[g(\eta z)\big],\qquad z\in\R .
\]
Empirically, $A$ is estimated by event studies: bin news events by signed magnitude $\times$ novelty $z$ and average the subsequent order-flow imbalance \citep{glasserman_new_2023}. The identification question: what does $A$ reveal about $Q$?

\subsection{Standing assumptions on the response function}

\begin{itemize}
\item[\textbf{(G1)}] $g:\R\to\R$ is odd, bounded, $C^2$ on $\R$, real-analytic on a neighborhood of $0$ with Taylor radius $r>0$ at the origin, $g'(0)>0$, and $g''<0$ on $(0,\infty)$ (hence $g$ is strictly concave there). Write $\cG$ for the class of such functions and $g(x)=\sum_{j\ge0}b_jx^{2j+1}$, $|x|<r$, for the (odd) Taylor expansion, $b_0=g'(0)$.
\item[\textbf{(G2)}] \emph{(Nonvanishing coefficients.)} $b_j\ne0$ for all $j\ge0$.
\item[\textbf{(Q1)}] $Q$ is a Borel probability measure on $[0,\bar\eta]$, $\bar\eta<\infty$, with $\E_Q[\eta]>0$ (equivalently $Q\ne\delta_0$: a nonnegative variable with zero mean is a.s.\ zero).
\end{itemize}

The canonical choice $g=\tanh$ (the output nonlinearity of the policy's news branch in the companion architecture) satisfies (G1) with $r=\pi/2$ and (G2): its expansion is $\tanh x=\sum_{n\ge1}\frac{2^{2n}(2^{2n}-1)B_{2n}}{(2n)!}x^{2n-1}$, so $b_j=\frac{2^{2j+2}(2^{2j+2}-1)B_{2j+2}}{(2j+2)!}$, and the Bernoulli numbers $B_{2n}$ are nonzero for all $n$; e.g.\ $b_0=1$, $b_1=-\tfrac13$.

Under (G1) and (Q1), $A$ is well defined, odd, bounded by $\sup|g|$, and, by dominated convergence with the local dominations $|g'|,|g''|$ bounded on compacts, $A\in C^2(\R)$ with
\begin{equation}
A'(z)=\E_Q\big[\eta\,g'(\eta z)\big],\qquad A''(z)=\E_Q\big[\eta^2g''(\eta z)\big].
\label{eq:aderivs}
\end{equation}
In particular $A'(0)=g'(0)\,\E_Q[\eta]>0$, and for $z>0$, $A''(z)<0$ whenever $Q((0,\infty))>0$, since the integrand is negative on $\{\eta>0\}$ and zero at $\eta=0$: $A$ itself is odd, bounded, and strictly concave on $(0,\infty)$.

\subsection{Detectability and identifiability}

\begin{theorem}[The impact curve identifies personal news sensitivity]
\label{thm:newsid}
Assume (G1) and (Q1), and write $m_{2j+1}=\E_Q[\eta^{2j+1}]$.
\begin{enumerate}
\item[(i)] (\emph{Detectability: strict Jensen gap.}) If $\Var_Q(\eta)>0$, then for every $z>0$
\[
A(z)\;<\;g\big(z\,\E_Q[\eta]\big).
\]
Consequently no homogeneous population reproduces $A$: if $A=A_{\delta_m}$ for some $m\ge0$, then matching origin slopes forces $m=\E_Q[\eta]$, and the displayed strict inequality is a contradiction at every $z>0$; hence $A$ determines whether $\Var_Q(\eta)=0$.
\item[(ii)] (\emph{Identifiability.}) Assume in addition (G2). Then $A$ restricted to any neighborhood of $0$ determines $Q$ completely: the odd moments $(m_{2j+1})_{j\ge0}$ are read off the Taylor coefficients of $A$, they determine $Q$ on $(0,\infty)$, and $Q(\{0\})=1-Q((0,\infty))$.
\end{enumerate}
\end{theorem}

\begin{proof}
(i) Fix $z>0$ and let $\varphi(x)=g(xz)$ on $[0,\bar\eta]$; $\varphi$ is concave on $[0,\bar\eta]$ (continuous, $\varphi''<0$ on $(0,\bar\eta]$) and not affine on any nondegenerate subinterval, since $\varphi''$ vanishes nowhere on $(0,\bar\eta]$. Let $m=\E_Q[\eta]\in(0,\bar\eta)$ (interior because $\Var_Q(\eta) > 0$ forces $Q \neq \delta_{\bar\eta}$ as well as $Q \neq \delta_0$). Concavity gives the supporting line $\varphi(x)\le\varphi(m)+\varphi'(m)(x-m)$ for all $x\in[0,\bar\eta]$. Equality at some $x_1\ne m$ would force $\varphi$ to be affine on the interval between $x_1$ and $m$ (a concave function pinched between a line at two points is affine in between), contradicting non-affineness. So the inequality is strict for $x\ne m$, and integrating against the non-degenerate $Q$,
\[
A(z)=\E_Q[\varphi(\eta)]<\varphi(m)+\varphi'(m)\,\E_Q[\eta-m]=\varphi(m)=g(zm).
\]
For the slope claim: by \eqref{eq:aderivs}, $A'(0)=g'(0)\E_Q[\eta]$, while $A_{\delta_m}'(0)=g'(0)m$; since $g'(0)>0$, equality of the curves near $0$ forces $m=\E_Q[\eta]$.

(ii) \emph{The Taylor coefficients.} Since the power series of $g$ converges absolutely for $|x|<r$, for $|z|<r/\bar\eta$ we have $\sum_{j\ge0}|b_j|\,\E_Q|\eta z|^{2j+1}\le\sum_j|b_j|(\bar\eta|z|)^{2j+1}<\infty$, so Fubini's theorem applies to $g(\eta z)=\sum_jb_j(\eta z)^{2j+1}$ (valid pointwise because $\eta|z|\le\bar\eta|z|<r$) and
\[
A(z)=\sum_{j\ge0}b_j\,m_{2j+1}\,z^{2j+1},\qquad|z|<r/\bar\eta .
\]
Thus $A$ is real-analytic near $0$ and its Taylor coefficients there are $b_jm_{2j+1}$; knowing $A$ on any neighborhood of $0$ determines these, and by (G2) determines every $m_{2j+1}$.

\emph{From odd moments to $Q$.} Let $\mu^*(d\eta):=\eta\,Q(d\eta)$, a finite Borel measure on $[0,\bar\eta]$ assigning no mass to $\{0\}$, and let $\nu$ be its pushforward under $\eta\mapsto\eta^2$, a finite measure on $[0,\bar\eta^2]$. Then for every $j\ge0$,
\[
\int u^j\,\nu(du)=\int\eta^{2j}\,\mu^*(d\eta)=\int\eta^{2j+1}\,Q(d\eta)=m_{2j+1}.
\]
A finite measure on a compact interval is determined by its moment sequence: if $\nu_1,\nu_2$ share all moments, they integrate all polynomials identically, hence (Stone--Weierstrass, plus boundedness of the interval) all continuous functions, hence $\nu_1=\nu_2$ by the Riesz representation theorem. (This is the determinate Hausdorff case of the classical moment problem \citep{shohat_problem_1943}.) So $\nu$ is determined by $A$; since $\eta\mapsto\eta^2$ is a homeomorphism of $[0,\infty)$ onto itself, $\mu^*$ is determined ($\mu^*$ is the pushforward of $\nu$ under $u\mapsto\sqrt u$); and $Q$ on $(0,\infty)$ is recovered by $Q(B)=\int_B\eta^{-1}\mu^*(d\eta)$ for Borel $B\subset(0,\infty)$, a finite integral because it equals $Q(B)\le1$. Finally $Q(\{0\})=1-Q((0,\infty))$.
\end{proof}

\medskip

\begin{remark}[If some coefficients vanish]
\label{rem:muntz}
(G2) can be weakened. Suppose $J=\{j\ge0:b_j\ne0\}$; note $0\in J$ always ($b_0=g'(0)>0$). The proof of (ii) recovers $\{m_{2j+1}:j\in J\}$, i.e.\ the $\nu$-moments $\{\int u^j\nu:j\in J\}$. If $\sum_{j\in J\setminus\{0\}}1/j=\infty$, then by the M\"untz theorem \citep{borwein_polynomials_1995} the span of $\{u^j:j\in J\}$ is dense in $C[0,\bar\eta^2]$, and the same Riesz argument identifies $\nu$, hence $Q$. Part (i) needs neither (G2) nor analyticity---only (G1)'s concavity and differentiability at $0$.
\end{remark}

\subsection{What survives when the curve is known only up to scale}

Empirically $A$ is estimated in order-flow units confounded with population size and order-size scaling, so only $\kappa A$ for an unknown $\kappa>0$ is observed. Identification degrades in an exactly describable way: the \emph{shape} of the active part of $Q$ survives; its total mass does not.

\begin{corollary}[Up-to-scale identification]
\label{cor:scale}
Assume (G1), (G2), (Q1). Let $Q,Q'$ satisfy (Q1) and suppose $A_{Q'}=\kappa A_Q$ for some $\kappa>0$. Then
\[
Q'\big(\cdot\mid\eta>0\big)=Q\big(\cdot\mid\eta>0\big)
\qquad\text{and}\qquad
Q'\big((0,\infty)\big)=\kappa\,Q\big((0,\infty)\big).
\]
In particular, from $A$ known up to scale one still identifies: (i) the conditional law of $\eta$ given $\eta>0$; and hence (ii) whether the \emph{active} population is heterogeneous, i.e.\ whether $\Var_Q(\eta\mid\eta>0)>0$. Equivalently: $A$ is proportional to $g(c\,\cdot)$ on a neighborhood of $0$ for some $c>0$ if and only if $Q(\cdot\mid\eta>0)=\delta_c$.
\end{corollary}

\begin{proof}
$A_{Q'}=\kappa A_Q$ near $0$ gives, by the coefficient identification in Theorem~\ref{thm:newsid}(ii) and (G2), $m_{2j+1}'=\kappa m_{2j+1}$ for all $j$, i.e.\ $\nu'=\kappa\nu$ (equal moment sequences on a compact interval, determinate as above), hence $\mu^{*\prime}=\kappa\mu^*$, hence $Q'|_{(0,\infty)}=\kappa\,Q|_{(0,\infty)}$ by the reconstruction formula. Two finite measures that are positive multiples of one another have the same normalization, which is the conditional-law claim; evaluating total masses gives $Q'((0,\infty))=\kappa Q((0,\infty))$ (in particular any such $Q'$ exists only for $\kappa\le1/Q((0,\infty))$).

For the equivalence: if $Q(\cdot\mid\eta>0)=\delta_c$, write $w=Q((0,\infty))$; then $A(z)=w\,g(cz)$, which is proportional to $g(c\,\cdot)$. Conversely if $A=\kappa g(c\,\cdot)$ near $0$, comparing coefficients gives $m_{2j+1}=\kappa c^{2j+1}$ for all $j$, so $\nu$ has moments $\kappa c\,(c^2)^j$, which are the moments of $\kappa c\,\delta_{c^2}$; by determinacy $\nu=\kappa c\,\delta_{c^2}$, so $\mu^*=\kappa c\,\delta_c$ and $Q|_{(0,\infty)}=\kappa\delta_c$: the active part is degenerate at $c$.
\end{proof}

\medskip

\begin{remark}[The news-ignorers are invisible, precisely]
Corollary~\ref{cor:scale} makes exact the informal caveat in the companion design: the atom $Q(\{0\})$ of news-ignoring agents affects $A$ only through overall scale, which is empirically confounded. What is testable from the shape of the impact curve alone is heterogeneity \emph{among agents who react at all}. The companion Study~A includes a $20\%$ $\eta=0$ component precisely to measure the practical cost of this confound.
\end{remark}

\subsection{The sharp boundary: unknown response functions}

Theorem~\ref{thm:newsid} treats $g$ as known---in PTMC it is, being fixed by the architecture. A skeptic may object that a different agent model, with a different response nonlinearity, could explain the same curve without heterogeneity. The skeptic is right, in the strongest possible sense; identification is a within-model-class property, and the following proposition delimits it exactly.

\begin{proposition}[Observational equivalence under unknown $g$]
\label{prop:unknowng}
Assume (G1) and (Q1), and let $A=A_{Q,g}$ be the impact curve of the (possibly heterogeneous) pair $(Q,g)$. Then for every $m>0$ the function $\tilde g:=A(\cdot/m)$ belongs to $\cG$, and the \emph{homogeneous} pair $(\delta_m,\tilde g)$ has the same impact curve:
\[
A_{\delta_m,\tilde g}(z)=\tilde g(mz)=A(z)\qquad\text{for all }z .
\]
Hence, if the response function is unknown within $\cG$, then not only is $Q$ unidentified---even the \emph{presence} of heterogeneity is undetectable from $A$: every heterogeneous population is observationally equivalent to a homogeneous one with a deformed response function. Under (G2)-type knowledge of $g$ up to scale, by contrast, Corollary~\ref{cor:scale}(ii) makes active-part heterogeneity detectable even with the curve known only up to scale.
\end{proposition}

\begin{proof}
The displayed identity is immediate from the definition. It remains to check $\tilde g\in\cG$, i.e.\ that $A(\cdot/m)$ satisfies (G1). Oddness and boundedness are inherited. $C^2$-smoothness and the derivative formulas are \eqref{eq:aderivs}. $\tilde g'(0)=A'(0)/m=g'(0)\E_Q[\eta]/m>0$ by (Q1). Strict concavity on $(0,\infty)$: for $z>0$, $A''(z)=\E_Q[\eta^2g''(\eta z)]<0$ because $Q((0,\infty))>0$ (by (Q1)) and the integrand is $<0$ there. Real-analyticity near $0$ with positive radius: by the series representation in the proof of Theorem~\ref{thm:newsid}(ii), $A$ has a convergent power series on $(-r/\bar\eta,\,r/\bar\eta)$.
\end{proof}

\medskip

\begin{remark}[What Proposition~\ref{prop:unknowng} does and does not say]
\label{rem:withinmodel}
It does \emph{not} undermine the companion design's test of heterogeneous news reaction: there, all competing arms share one architecture---one $g$---and differ only in the persona distribution, which is exactly the known-$g$ regime of Theorem~\ref{thm:newsid} and Corollary~\ref{cor:scale}. It \emph{does} say that a claim of the form ``order-flow data demonstrate heterogeneous news sensitivity, whatever the response function'' is unprovable from the impact curve alone: the data constrain the pair $(g,Q)$, and the heterogeneity conclusion is conditional on the response family. Any empirical write-up should state the conclusion in this conditional form. (One referee-proof reading: the curve identifies the \emph{mixture} $\int g(\eta\,\cdot)\,Q(d\eta)$; splitting it into $g$ and $Q$ requires fixing one of them.)
\end{remark}

\begin{remark}[Scalar identifiability versus a multivariate persona]
\label{rem:multivariate}
Theorem~\ref{thm:newsid} identifies a single scalar coordinate, $\eta$, of the persona $p=(\theta,\rho)$, holding the rest of the persona and the response nonlinearity $g$ fixed. The companion framework's personas are typically multivariate---news sensitivity $\eta$, loss aversion $\lambda$, a herding propensity $h$, and further behavioral coordinates---and nothing in this paper proves \emph{joint} identifiability of their distribution from a single aggregate response curve. The obstruction is not merely notational: recovering a multivariate mixing distribution from a low-dimensional aggregate functional is a substantially harder inverse problem than the one-dimensional moment problem of Theorem~\ref{thm:newsid}(ii). The analogous multivariate moment problem is, in general, indeterminate without further structure (e.g.\ independence, or a known copula linking the coordinates), and even where it is determinate, nonparametric recovery rates in the random-coefficients literature degrade sharply with dimension \citep{gautier_nonparametric_2013}. Two routes avoid this curse in practice, at a cost: (i) an ``orthogonal shocks'' design, in which distinct empirical instruments are each chosen to isolate one coordinate---news-magnitude events for $\eta$, disposition-effect event studies for $\lambda$, order-flow clustering for $h$---holding the others fixed, reducing the joint problem to a sequence of one-dimensional ones of exactly the type solved here; or (ii) accepting only partial identification of the joint law. The identification theory of this paper should be read as settling the one-dimensional case cleanly and exactly---the case the companion Study~A isolates by design---not as a template that extends automatically to the full persona vector.
\end{remark}

\section{Estimation and testing of the sensitivity distribution}
\label{sec:inference}

Theorem~\ref{thm:newsid} is a population statement about exact knowledge of $A$ near $0$. This section provides the corresponding finite-sample instruments: a parametric $\sqrt n$ estimation theorem, and a test of homogeneity that is correct at the boundary. Throughout, data are event-study observations
\[
(z_i,Y_i),\quad i=1,\dots,n,\ \text{i.i.d.},\qquad Y_i=A_{\vartheta_0}(z_i)+\epsilon_i,\quad\E[\epsilon_i\mid z_i]=0,
\]
where $z_i$ is the signed news magnitude of event $i$ (law $\lambda$, compactly supported), $Y_i$ the post-event order-flow response, $|Y_i| \le C_Y$ bounded (order flow is capped, as in Section~\ref{sec:setup}), and $\{A_\vartheta:\vartheta\in\Theta\}$ the impact curves of a parametric family $\{Q_\vartheta\}$.

\begin{itemize}
\item[\textbf{(R1)}] $\Theta\subset\R^{d_\vartheta}$ is compact; $\vartheta\mapsto Q_\vartheta$ is injective; every $Q_\vartheta$ is supported in $[0,\bar\eta]$ and satisfies (Q1); $g$ satisfies (G1)--(G2).
\item[\textbf{(R2)}] For each $z$ in the support of $\lambda$, $\vartheta\mapsto A_\vartheta(z)$ is continuous on $\Theta$ and twice continuously differentiable on a neighborhood of $\vartheta_0$, with $\sup_{z,\vartheta}(|A_\vartheta(z)|+\|\nabla_\vartheta A_\vartheta(z)\|+\|\nabla^2_\vartheta A_\vartheta(z)\|)<\infty$ on that neighborhood.
\item[\textbf{(R3)}] The support of $\lambda$ has an accumulation point in $(0,\,r/\bar\eta)$.
\item[\textbf{(R4)}] $\vartheta_0$ is interior to $\Theta$ and $J:=\E_\lambda\big[\nabla_\vartheta A_{\vartheta_0}(z)\,\nabla_\vartheta A_{\vartheta_0}(z)^{\!\top}\big]$ is nonsingular.
\end{itemize}

Define the nonlinear least-squares estimator $\hat\vartheta_n=\argmin_{\vartheta\in\Theta}Q_n(\vartheta)$, $Q_n(\vartheta)=\frac1n\sum_{i=1}^n\big(Y_i-A_\vartheta(z_i)\big)^2$.

\begin{lemma}[Identification from the design]
\label{lem:ident}
Under (R1) and (R3): if $A_\vartheta=A_{\vartheta_0}$ $\lambda$-almost everywhere, then $\vartheta=\vartheta_0$.
\end{lemma}

\begin{proof}
$D:=\{z:A_\vartheta(z)\ne A_{\vartheta_0}(z)\}$ is open (both curves are continuous) and $\lambda$-null, hence disjoint from $\mathrm{supp}\,\lambda$; so the two curves agree on $\mathrm{supp}\,\lambda$, whose accumulation point $z^*\in(0,r/\bar\eta)$ lies in the interior of the interval $(-r/\bar\eta,r/\bar\eta)$ on which both curves are real-analytic (proof of Theorem~\ref{thm:newsid}(ii)). An analytic function whose zero set has an accumulation point in its (connected) domain vanishes identically; hence $A_\vartheta=A_{\vartheta_0}$ on a neighborhood of $0$, and Theorem~\ref{thm:newsid}(ii) gives $Q_\vartheta=Q_{\vartheta_0}$, so $\vartheta=\vartheta_0$ by injectivity.
\end{proof}

\medskip

\begin{theorem}[$\sqrt n$-consistent estimation of the sensitivity distribution]
\label{thm:nls}
Under (R1)--(R4), $\hat\vartheta_n\to\vartheta_0$ in probability, and
\[
\sqrt n\,\big(\hat\vartheta_n-\vartheta_0\big)\;\Rightarrow\;\mathcal N\Big(0,\;J^{-1}\,\E\big[\sigma^2(z)\,\nabla A_{\vartheta_0}(z)\nabla A_{\vartheta_0}(z)^{\!\top}\big]\,J^{-1}\Big),
\qquad\sigma^2(z):=\E[\epsilon^2\mid z].
\]
In the two-parameter case $\vartheta=(\E\eta,\Var\eta)$ ranging over a family satisfying (R1)--(R2), this is the ``two-moment'' estimator of the companion design, with sandwich standard errors.
\end{theorem}

\begin{proof}
\emph{Consistency.} The criterion class $\{(y,z)\mapsto(y-A_\vartheta(z))^2:\vartheta\in\Theta\}$ is a family of bounded functions (uniform bound from $|Y| \le C_Y$ and (R2)), continuous in $\vartheta$ for each $(y,z)$, indexed by a compact set; by the uniform law of large numbers for such classes \citep[Lemma~2.4]{newey_large_1994}, $\sup_\Theta|Q_n(\vartheta)-M(\vartheta)|\to0$ in probability, where
\[
M(\vartheta)=\E\big[(Y-A_\vartheta(z))^2\big]=\E[\sigma^2(z)]+\E\big[(A_{\vartheta_0}(z)-A_\vartheta(z))^2\big].
\]
$M$ is continuous, and by Lemma~\ref{lem:ident} it is uniquely minimized at $\vartheta_0$. Standard M-estimation consistency \citep[Thm.~2.1]{newey_large_1994} yields $\hat\vartheta_n\to_p\vartheta_0$.

\emph{Asymptotic normality.} Conditions of \citet[Thm.~3.1]{newey_large_1994} (equivalently \citealp[Thm.~5.23]{vandervaart_asymptotic_1998}): $\vartheta_0$ interior (R4); $Q_n$ twice continuously differentiable near $\vartheta_0$ with
$\nabla Q_n(\vartheta_0)=-\frac2n\sum_i\epsilon_i\nabla A_{\vartheta_0}(z_i)$, an i.i.d.\ average of mean-zero bounded vectors, so $\sqrt n\,\nabla Q_n(\vartheta_0)\Rightarrow\mathcal N(0,\,4\,\E[\sigma^2(z)\nabla A\nabla A^{\!\top}])$ by the multivariate CLT; and
$\nabla^2Q_n(\vartheta)\to_p2\,\E[\nabla A_\vartheta\nabla A_\vartheta^{\!\top}]-2\,\E[(Y-A_\vartheta)\nabla^2A_\vartheta]$ uniformly near $\vartheta_0$ (ULLN again, using the (R2) dominations), which at $\vartheta_0$ equals $2J\succ0$ since $\E[\epsilon\,\nabla^2A_{\vartheta_0}(z)]=0$. The sandwich formula follows:
$\sqrt n(\hat\vartheta_n-\vartheta_0)\Rightarrow\mathcal N\big(0,(2J)^{-1}\,4\,\E[\sigma^2\nabla A\nabla A^{\!\top}]\,(2J)^{-1}\big)$, which is the display.
\end{proof}

\medskip

\subsection{Testing homogeneity: a null on the boundary}
\label{sec:boundarytest}

The scientific null---``news reaction is \emph{not} personal''---is $H_0:\Var(\eta)=0$, a boundary point of any parameterization in which the variance is a parameter. We work in the natural two-parameter family
\begin{equation}
\eta_\theta\;=\;m+\sqrt v\,\xi,\qquad\theta=(m,v)\in\Theta\subset\R\times[0,\infty),
\label{eq:family}
\end{equation}
where $\xi$ is a fixed mean-zero, unit-variance random variable with compact support and $\E|\xi|^3<\infty$, and $\Theta$ is a compact neighborhood of $(m_0,0)$ in $\R\times[0,\infty)$, small enough that $\eta_\theta\in[0,\bar\eta]$ for all $\theta\in\Theta$; the true parameter is $\theta_0=(m_0,0)$ with $m_0>0$ interior in its coordinate. (Any family smoothly parameterized by mean and variance with a fixed standardized shape works identically; \eqref{eq:family} is the concrete one the companion design fits.) We strengthen the smoothness of $g$ to $C^3$---satisfied by $\tanh$---and assume homoskedastic errors for the likelihood-ratio-type statistic:
\begin{itemize}
\item[\textbf{(B1)}] $g\in\cG$ is $C^3$ on $\R$; $(z_i,Y_i)$ as above with $\E[\epsilon^2\mid z]=\sigma^2>0$ constant and $|\epsilon|$ bounded; (R3) holds; and the two functions $z\mapsto z\,g'(m_0z)$ and $z\mapsto\tfrac12z^2g''(m_0z)$ are linearly independent in $L^2(\lambda)$, so that $J=\E_\lambda[\nabla A_{\theta_0}\nabla A_{\theta_0}^{\!\top}]\succ0$, where $\nabla A_{\theta_0}(z):=\big(zg'(m_0z),\,\tfrac12z^2g''(m_0z)\big)^{\!\top}$ is the (one-sided in $v$) gradient computed in Lemma~\ref{lem:expansion} below.
\end{itemize}
For $g=\tanh$, linear independence holds for every $m_0>0$ whenever $\mathrm{supp}\,\lambda$ has an accumulation point in $(0,r/\bar\eta)$: a relation $a\,zg'(m_0z)+\tfrac b2z^2g''(m_0z)\equiv0$ on an interval forces, by comparing the $z$ and $z^3$ Taylor coefficients ($g'(x)=1-x^2+O(x^4)$, $g''(x)=-2x+O(x^3)$), first $a=0$ and then $b\,m_0=0$, so $a=b=0$.

Define the restricted and unrestricted estimators and the quasi-likelihood-ratio statistic
\[
\tilde\theta_n=\argmin_{\Theta\cap\{v=0\}}Q_n,\qquad
\hat\theta_n=\argmin_{\Theta}Q_n,\qquad
T_n\;=\;\frac{n\,\big(Q_n(\tilde\theta_n)-Q_n(\hat\theta_n)\big)}{\hat\sigma_n^2},\quad
\hat\sigma_n^2:=Q_n(\hat\theta_n).
\]

\begin{lemma}[Smoothness of the family at the boundary]
\label{lem:expansion}
Under (B1), for $\theta=(m,v)\in\Theta$ and $z$ in the (compact) support of $\lambda$:
\begin{enumerate}
\item $A_\theta(z)=g(mz)+\dfrac v2\,z^2g''(mz)+R(\theta,z)$ with $|R(\theta,z)|\le C\,v^{3/2}$ uniformly;
\item $\partial A_\theta(z)/\partial v$ exists on $\{v>0\}$, extends continuously to $v=0$ with value $\tfrac12z^2g''(mz)$, and $\big|\partial_vA_\theta(z)-\tfrac12z^2g''(mz)\big|\le C\sqrt v$ uniformly; $\partial_mA_\theta(z)$ is Lipschitz in $\theta$;
\item consequently $A_\theta(z)-A_{\theta_0}(z)=\nabla A_{\theta_0}(z)^{\!\top}(\theta-\theta_0)+\rho(\theta,z)$ with $|\rho(\theta,z)|\le C\|\theta-\theta_0\|^{3/2}$ uniformly, and $A_\theta(z)$ is Lipschitz in $\theta$ on $\Theta$, uniformly in $z$.
\end{enumerate}
\end{lemma}

\begin{proof}
(1) Taylor's theorem with Lagrange remainder in the scalar $u=\sqrt v$: for fixed $\xi$-realization and $z$,
\[
g\big((m+u\xi)z\big)=g(mz)+u\,\xi z\,g'(mz)+\frac{u^2\xi^2}{2}z^2g''(mz)+\frac{u^3\xi^3}{6}z^3g'''\big((m+\tau u\xi)z\big),\ \tau\in(0,1).
\]
Take expectations: the linear term vanishes ($\E\xi=0$), the quadratic term gives $\tfrac v2z^2g''(mz)$ ($\E\xi^2=1$), and the remainder is bounded by $\tfrac{v^{3/2}}{6}\E|\xi|^3\,\bar z^3\sup|g'''|$ over the compact range of arguments.

(2) For $v>0$, differentiating under the (finite-range) expectation, $\partial_vA=\E\big[g'((m+\sqrt v\xi)z)\,z\xi\big]/(2\sqrt v)$. Expanding $g'((m+\sqrt v\xi)z)=g'(mz)+\sqrt v\,\xi z\,g''(mz)+\tfrac{v\xi^2z^2}{2}g'''(\cdot)$ and using $\E\xi=0$,
\[
\partial_vA=\frac{z}{2\sqrt v}\Big[\sqrt v\,z\,g''(mz)\,\E\xi^2+O(v)\Big]=\frac{z^2}{2}g''(mz)+O(\sqrt v),
\]
with the $O$-terms uniform by compactness. $\partial_mA=\E[z\,g'((m+\sqrt v\xi)z)]$ is $C^1$ in $(m,\sqrt v)$ with bounded derivatives, hence Lipschitz in $\theta$ (in $v$: $|\partial_v\partial_mA|=|\E[z^2\xi g''(\cdot)]/(2\sqrt v)|=O(1)$ by the same cancellation).

(3) Integrate the gradient along the segment from $\theta_0$ to $\theta$ (staying in $\R\times[0,\infty)$, where the one-sided derivative at $v=0$ makes the fundamental theorem of calculus valid coordinatewise) and use (2): the $m$-coordinate contributes a Lipschitz-gradient (hence $O(\|\Delta\|^2)$) error, the $v$-coordinate a H\"older-$\tfrac12$-gradient error $O(|\Delta_v|\cdot|\Delta_v|^{1/2})$; both are $O(\|\theta-\theta_0\|^{3/2})$. Lipschitz continuity of $A_\theta$ in $\theta$ follows from the boundedness of both partials.
\end{proof}

\medskip

\begin{theorem}[Boundary test of homogeneous news reaction]
\label{thm:boundary}
Under (B1) and $H_0:\theta_0=(m_0,0)$,
\[
T_n\;\Rightarrow\;\tfrac12\chi^2_0+\tfrac12\chi^2_1,
\]
where the limit is the law of $W^2\,\mathbf 1\{W>0\}$ for $W\sim\mathcal N(0,1)$: an atom of mass $\tfrac12$ at $0$, and conditionally on being positive a $\chi^2_1$ distribution. The test rejecting when $T_n>\chi^2_{1,1-2\alpha}$ has asymptotic size $\alpha$ for $\alpha<\tfrac12$.
\end{theorem}

\begin{proof}
Write $\Delta=\theta-\theta_0$ and $h=\sqrt n\,\Delta$. All statements are under $H_0$.

\emph{Step 1: consistency and rate.} $M(\theta)-M(\theta_0)=\E_\lambda[(A_\theta-A_{\theta_0})^2]$. By Lemma~\ref{lem:ident} (whose proof needs only continuity, analyticity near $0$ of each curve---valid for the family \eqref{eq:family} since each $\eta_\theta$ has compact support---and Theorem~\ref{thm:newsid}(ii) applied to the measures $Q_\theta$, plus injectivity of $\theta\mapsto Q_\theta$, which holds because mean and variance are determined by $Q_\theta$), $\theta_0$ is the unique minimizer of $M$ on $\Theta$; with the ULLN as in Theorem~\ref{thm:nls}, both $\hat\theta_n$ and $\tilde\theta_n$ are consistent. For the rate: by Lemma~\ref{lem:expansion}(3),
\[
\big\|A_\theta-A_{\theta_0}\big\|_{L^2(\lambda)}\;\ge\;\big\|\nabla A_{\theta_0}^{\!\top}\Delta\big\|_{L^2(\lambda)}-C\|\Delta\|^{3/2}
\;\ge\;\Big(\sqrt{\lambda_{\min}(J)}-C\|\Delta\|^{1/2}\Big)\|\Delta\|,
\]
so $M(\theta)-M(\theta_0)\ge c\|\Delta\|^2$ on a neighborhood of $\theta_0$. The criterion difference class $\{m_\theta-m_{\theta_0}:\|\Delta\|\le\delta\}$, $m_\theta(y,z)=(y-A_\theta(z))^2$, satisfies $|m_\theta-m_{\theta'}|\le C'\|\theta-\theta'\|$ (Lipschitz in $\theta$ with constant envelope, by Lemma~\ref{lem:expansion}(3) and boundedness of $Y$); a Lipschitz-in-parameter class over a subset of $\R^2$ has bracketing entropy of parametric order, so the empirical-process modulus condition of the rate theorem \citep[Thm.~5.52]{vandervaart_asymptotic_1998} holds with $\phi(\delta)=C\delta$ (via the standard bracketing maximal inequality, \citealp[Ch.~19]{vandervaart_asymptotic_1998}), and that theorem gives $\|\hat\theta_n-\theta_0\|=O_P(n^{-1/2})$, and the same for $\tilde\theta_n$ (the restricted problem is the same argument on the sub-parameter-space $\{v=0\}$, where $\theta_0$ is interior in the induced topology of the $m$-axis). Note that the rate theorem does not require $\theta_0$ to be interior in $\R^2$; it requires only the quadratic lower bound and the modulus condition, both established.

\emph{Step 2: uniform quadratic expansion.} Let
$Z_n:=\frac{1}{\sqrt n}\sum_i\epsilon_i\nabla A_{\theta_0}(z_i)\Rightarrow Z\sim\mathcal N(0,\sigma^2J)$ (multivariate CLT; summands bounded, mean zero) and $J_n:=\frac1n\sum_i\nabla A_{\theta_0}(z_i)\nabla A_{\theta_0}(z_i)^{\!\top}\to_pJ$. Fix $H<\infty$. For $\|h\|\le H$ and $\theta=\theta_0+h/\sqrt n\in\Theta$, expand, writing $A_\theta-A_{\theta_0}=\nabla A_{\theta_0}^{\!\top}h/\sqrt n+\rho_n$ with $|\rho_n|\le C(H/\sqrt n)^{3/2}=CH^{3/2}n^{-3/4}$ (Lemma~\ref{lem:expansion}(3)):
\begin{align*}
n\big[Q_n(\theta)-Q_n(\theta_0)\big]
&=-\frac{2}{\sqrt n}\sum_i\epsilon_i\Big[\sqrt n\,(A_\theta-A_{\theta_0})(z_i)\Big]+\sum_i(A_\theta-A_{\theta_0})^2(z_i)\\
&=-2Z_n^{\!\top}h+h^{\!\top}J_nh+r_n(h),
\end{align*}
where $r_n$ collects three terms: (i) $-2\sum_i\epsilon_i\rho_n(z_i)$, whose conditional-on-$(z_i)$ variance is at most $4n\sigma^2C^2H^3n^{-3/2}\to0$, so it is $o_P(1)$; (ii) the cross term $2\sum_i(\nabla A^{\!\top}h/\sqrt n)\rho_n(z_i)$, bounded by $2n\cdot C_1Hn^{-1/2}\cdot CH^{3/2}n^{-3/4}=O(H^{5/2}n^{-1/4})\to0$; (iii) $\sum_i\rho_n(z_i)^2\le nC^2H^3n^{-3/2}\to0$. All bounds are uniform over $\|h\|\le H$, so
\begin{equation}
\sup_{\|h\|\le H}\Big|n\big[Q_n(\theta_0+h/\sqrt n)-Q_n(\theta_0)\big]-q_n(h)\Big|\to_p0,
\qquad q_n(h):=-2Z_n^{\!\top}h+h^{\!\top}J_nh .
\label{eq:quad}
\end{equation}

\emph{Step 3: cone projection.} The local parameter sets are $\sqrt n(\Theta-\theta_0)\cap\{\|h\|\le H\}$, which converge (locally in Hausdorff distance, since $m_0$ is interior in its coordinate and $\Theta$ is a neighborhood of $\theta_0$ in $\R\times[0,\infty)$) to the closed convex cone $\Lambda=\R\times[0,\infty)$, and $\{v=0\}$ gives $\Lambda_0=\R\times\{0\}$. By Step~1, $\sqrt n(\hat\theta_n-\theta_0)$ and $\sqrt n(\tilde\theta_n-\theta_0)$ are $O_P(1)$, so with probability tending to one they lie in $\{\|h\|\le H\}$ for large $H$; \eqref{eq:quad} plus the continuity and strict convexity of the limiting quadratic $q(h)=-2Z^{\!\top}h+h^{\!\top}Jh$ (recall $J\succ 0$) then give, by the argmin continuous-mapping argument for convex quadratics over closed convex cones (the quadratic-approximation framework of \citealp{andrews_testing_2001}; in the likelihood setting \citealp{chernoff_distribution_1954,self_asymptotic_1987}),
\[
n\big(Q_n(\tilde\theta_n)-Q_n(\hat\theta_n)\big)\;\Rightarrow\;\min_{h\in\Lambda_0}q(h)-\min_{h\in\Lambda}q(h).
\]
Compute the right side. Write $Z=(Z_m,Z_v)$, $\hat h=J^{-1}Z$ the unconstrained minimizer, $q(\hat h)=-Z^{\!\top}J^{-1}Z$. Over $\Lambda_0$: minimizing $-2Z_mh_m+J_{mm}h_m^2$ gives $-Z_m^2/J_{mm}$. Over $\Lambda$: if $\hat h_v\ge0$ the unconstrained minimizer is feasible; otherwise the minimum over the convex set lies on the face $\{v=0\}$ and coincides with the $\Lambda_0$ minimum. The partitioned-inverse identity
\[
Z^{\!\top}J^{-1}Z-\frac{Z_m^2}{J_{mm}}=\frac{\big((J^{-1}Z)_v\big)^2}{(J^{-1})_{vv}}
\]
(direct computation: with $J=\left(\begin{smallmatrix}a&b\\b&c\end{smallmatrix}\right)$, $Z=(x,y)$, both sides equal $(bx-ay)^2/\big(a(ac-b^2)\big)$) gives
\[
\min_{\Lambda_0}q-\min_{\Lambda}q=\frac{U^2}{(J^{-1})_{vv}}\,\mathbf 1\{U>0\},\qquad U:=(J^{-1}Z)_v\sim\mathcal N\big(0,\sigma^2(J^{-1})_{vv}\big).
\]
So $n(Q_n(\tilde\theta_n)-Q_n(\hat\theta_n))\Rightarrow\sigma^2W^2\mathbf 1\{W>0\}$ with $W\sim\mathcal N(0,1)$. Finally $\hat\sigma_n^2=Q_n(\hat\theta_n)\to_p\sigma^2$ (from consistency and the ULLN), so by Slutsky $T_n\Rightarrow W^2\mathbf 1\{W>0\}$, which is the $\tfrac12\chi^2_0+\tfrac12\chi^2_1$ mixture ($\Prob(W\le0)=\tfrac12$). The size statement follows since $\Prob(W^2\mathbf 1\{W>0\}>\chi^2_{1,1-2\alpha})=\tfrac12\cdot2\alpha=\alpha$.
\end{proof}

\medskip

\begin{remark}[Practical variants and a bootstrap warning]
\label{rem:bootwarn}
(i) Heteroskedastic errors: replace $T_n$ by the one-sided $t$-statistic $\hat v_n/\widehat{\mathrm{se}}(\hat v_n)$ with sandwich standard error; its null limit is $\max(0,W)$ for standard normal $W$ by the same Steps~1--3 (the constrained limiting minimizer's $v$-coordinate is $\max(0,\hat h_v)$), so one rejects at level $\alpha$ when it exceeds $z_{1-\alpha}$. (ii) Do \emph{not} bootstrap the naive way: resampling-based critical values for a parameter on the boundary are inconsistent \citep{andrews_inconsistency_2000}. The companion design's original suggestion to ``test $H_0:\Var(\eta)=0$ via the bootstrap'' should be implemented, if at all, with boundary-aware modifications (e.g.\ shrinking-parameter or $m$-out-of-$n$ schemes); the asymptotic mixture of Theorem~\ref{thm:boundary} is the simpler instrument and is exactly what the $\tfrac12\chi^2_0+\tfrac12\chi^2_1$ tables of \citet{self_asymptotic_1987} were made for. (iii) The dose--response arm of the companion design ($\sigma_\eta$ swept away from the calibrated value) provides power curves for exactly this test.
\end{remark}

\section{Separation theorems: when PTMC is provably the better tool}
\label{sec:separation}

No theorem can assert that one modeling framework dominates all others on all problems; for any two methods there are data-generating processes favoring either. What can be proved are \emph{conditional separations}: fix a target quantity and an information set, and show that PTMC achieves an error that provably no method of a rival class can achieve. This section proves two such results. Theorem~\ref{thm:quantjensen} and its uniform version Theorem~\ref{thm:homsep} give an \emph{irreducible asymptotic bias floor}, proportional to the sensitivity variance, for \emph{every} homogeneous-population simulator---the quantitative form of the Jensen gap of Theorem~\ref{thm:newsid}(i). Theorem~\ref{thm:lucas} is stronger in scope: for \emph{interventional} targets, no estimator that is a functional of observational data---a class containing Bayesian-neural-network forecasters, dropout ensembles, and LLM-based predictors alike---can have worst-case error below an explicit constant, at any sample size, while the structural PTMC estimator is consistent when its model class is correctly specified. The second result is the Lucas critique \citep{lucas_econometric_1976} converted from a methodological warning into a minimax bound, using this paper's own observational-equivalence construction (Proposition~\ref{prop:unknowng}) as the two-point pair. Read against the current causal-machine-learning literature, PTMC occupies a specific position: the policy $\pi_\phi$ supplies the predictive power of behavioral cloning, fit directly from data, while the explicit persona distribution $\cP$ supplies a handle for intervention---set $\cP$ to a counterfactual population and simulate forward. These are not competing sources of the same capability; Theorem~\ref{thm:lucas} is the formal statement that the first cannot substitute for the second, however large the model or the dataset behind it.

\subsection{A bias floor for homogeneous-population simulators}

Theorem~\ref{thm:newsid}(i) shows the representative-agent impact curve is wrong whenever $\Var_Q(\eta)>0$; the next theorem shows \emph{by how much}, with an explicit constant.

\begin{theorem}[Quantitative Jensen gap]
\label{thm:quantjensen}
Assume (G1), (Q1), and $\Var_Q(\eta)>0$, and write $m=\E_Q[\eta]$. Then for every $z>0$,
\[
g(zm)-A(z)\;\ge\;\tfrac38\,z^2\,\beta(z)\,\Var_Q(\eta),
\qquad
\beta(z):=\min_{y\in[\,mz/2,\;\bar\eta z\,]}\big|g''(y)\big|\;>\;0 .
\]
\end{theorem}

\begin{proof}
Fix $z>0$ and let $\psi(x):=g(zm)+zg'(zm)(x-m)-g(zx)$ for $x\in[0,\bar\eta]$, so that $g(zm)-A(z)=\E_Q[\psi(\eta)]$ (the linear term integrates to zero). Taylor's theorem with integral remainder gives, for every $x$,
\[
\psi(x)\;=\;\int_{x\wedge m}^{\,x\vee m}\big|x-u\big|\;\big(-z^2g''(zu)\big)\,du\;\ge\;0 .
\]
\emph{Case $x\ge m$.} The integration variable satisfies $u\in[m,\bar\eta]$, so $zu\in[mz,\bar\eta z]\subset[mz/2,\bar\eta z]$ and $-g''(zu)\ge\beta(z)$; hence
$\psi(x)\ge z^2\beta(z)\int_m^x(x-u)\,du=\tfrac12z^2\beta(z)(x-m)^2\ge\tfrac38z^2\beta(z)(x-m)^2$.

\emph{Case $x<m$.} Discard the part of the integral below $u=(x+m)/2$; since $x\ge0$ implies $(x+m)/2\ge m/2$, on the retained range $zu\in[mz/2,\,mz]\subset[mz/2,\bar\eta z]$, so
\[
\psi(x)\;\ge\;z^2\beta(z)\int_{(x+m)/2}^{m}(u-x)\,du
\;=\;z^2\beta(z)\cdot\tfrac12\Big[(m-x)^2-\big(\tfrac{m-x}{2}\big)^2\Big]
\;=\;\tfrac38\,z^2\beta(z)(m-x)^2 .
\]
Combining the cases, $\psi(x)\ge\tfrac38z^2\beta(z)(x-m)^2$ on all of $[0,\bar\eta]$; taking $\E_Q$ gives the display. Positivity of $\beta(z)$: $g''$ is continuous and strictly negative on the compact interval $[mz/2,\bar\eta z]\subset(0,\infty)$, and $m>0$ by (Q1).
\end{proof}

\medskip

\begin{remark}
For $g=\tanh$, $|g''(y)|=2\tanh y\,(1-\tanh^2y)$ is unimodal on $(0,\infty)$, so $\beta(z)$ is attained at an endpoint of $[mz/2,\bar\eta z]$ and is available in closed form. Theorem~\ref{thm:quantjensen} directly prices the representative-persona arm of the companion horse race (its arm~(b)): that arm's impact-curve bias at news magnitude $z$ is at least $\tfrac38z^2\beta(z)\Var_Q(\eta)$, however well it is calibrated otherwise.
\end{remark}

A homogeneous simulator is not obliged to match the slope at the origin, and it may include news-ignoring agents; the honest comparison class is therefore every curve $w\,g(c\,\cdot)$ with active mass $w\in[0,1]$ and sensitivity $c$ in the persona space. Separation survives, uniformly.

\begin{theorem}[Uniform separation from all homogeneous populations]
\label{thm:homsep}
Assume (G1), (G2), (Q1), $\Var_Q(\eta\mid\eta>0)>0$, and fix $[z_1,z_2]\subset(0,\,r/\bar\eta)$. Then
\[
\delta\;:=\;\inf_{w\in[0,1],\;c\in[0,\bar\eta]}\;\sup_{z\in[z_1,z_2]}\big|w\,g(cz)-A(z)\big|\;>\;0 .
\]
\end{theorem}

\begin{proof}
First note that (G1) forces $g'>0$ on $(0,\infty)$: $g'$ is strictly decreasing there ($g''<0$), and if $g'(x_0)\le0$ for some $x_0$ then $g'(x)<g'(x_0+\tfrac{x_0}{2})<0$ for all large $x$, whence $g(x)\to-\infty$, contradicting boundedness. Hence $g>0$ on $(0,\infty)$ and, by \eqref{eq:aderivs} and (Q1), $A(z_1)>0$.

The map $(w,c)\mapsto\sup_{z\in[z_1,z_2]}|wg(cz)-A(z)|$ is continuous on the compact set $[0,1]\times[0,\bar\eta]$ ($g$ is uniformly continuous on $[0,\bar\eta z_2]$, and a supremum of jointly continuous functions over a compact $z$-interval is continuous), so the infimum is attained at some $(w^*,c^*)$. Suppose $\delta=0$, i.e.\ $w^*g(c^*z)=A(z)$ on $[z_1,z_2]$. If $w^*=0$ or $c^*=0$ the left side vanishes identically, contradicting $A(z_1)>0$; so $w^*>0$, $c^*\in(0,\bar\eta]$. Both sides are real-analytic on the interval $(-r/\bar\eta,\,r/\bar\eta)$: $A$ by the series representation in the proof of Theorem~\ref{thm:newsid}(ii), and $z\mapsto g(c^*z)$ because $|c^*z|\le\bar\eta|z|<r$ there. Their difference vanishes on $[z_1,z_2]$, a set with accumulation points interior to the (connected) domain, hence vanishes identically; in particular $A=w^*g(c^*\cdot)$ on a neighborhood of $0$. By the equivalence in Corollary~\ref{cor:scale}, this forces $Q(\cdot\mid\eta>0)=\delta_{c^*}$, contradicting $\Var_Q(\eta\mid\eta>0)>0$.
\end{proof}

\medskip

\begin{corollary}[Asymptotic risk dominance]
\label{cor:homsepmse}
In the setting of Theorem~\ref{thm:homsep}, call an estimation method \emph{homogeneous-population} if its fitted impact curve $\widehat A_n$ satisfies $\inf_{w\in[0,1],c\in[0,\bar\eta]}\sup_{z\in[z_1,z_2]}|\widehat A_n(z)-wg(cz)|\to_p0$ (its output lies asymptotically in the homogeneous family, as it does for any procedure that fits $(w,c)$ by any criterion). Then
\[
\Prob\Big(\sup_{z\in[z_1,z_2]}\big|\widehat A_n(z)-A(z)\big|\ge\tfrac\delta2\Big)\;\to\;1 ,
\]
an error floor no data or compute can remove. By contrast, under the conditions of Theorem~\ref{thm:nls} with a persona family containing $Q$, the PTMC plug-in curve $A_{\hat\vartheta_n}$ satisfies $\sup_{z\in[z_1,z_2]}|A_{\hat\vartheta_n}(z)-A(z)|\to_p0$.
\end{corollary}

\begin{proof}
For any $(w,c)$, the triangle inequality gives $\|\widehat A_n-A\|_\infty\ge\|wg(c\cdot)-A\|_\infty-\|\widehat A_n-wg(c\cdot)\|_\infty\ge\delta-\|\widehat A_n-wg(c\cdot)\|_\infty$ (sup norms on $[z_1,z_2]$); taking the infimum over $(w,c)$ and using the hypothesis yields $\|\widehat A_n-A\|_\infty\ge\delta-o_p(1)$. For the PTMC side: $\hat\vartheta_n\to_p\vartheta_0$ (Theorem~\ref{thm:nls}), and by (R2) the map $\vartheta\mapsto A_\vartheta(z)$ is continuous uniformly in $z\in[z_1,z_2]$ (dominated gradients), so the plug-in curve converges uniformly in probability.
\end{proof}

\medskip

\subsection{An impossibility theorem for reduced-form methods under intervention}
\label{sec:lucasbound}

The deeper comparison is not with homogeneous simulators but with the modern reduced-form toolkit: Bayesian neural networks, deep ensembles, dropout uncertainty, LLM-based forecasters. These differ in architecture but share one structural property: their output is a (possibly randomized) functional of \emph{observational data}. The following theorem shows that for interventional targets this property is itself the binding constraint, regardless of architecture, sample size, or compute. We work in the observation scheme of Section~\ref{sec:inference}, made explicit as an assumption:

\begin{itemize}
\item[\textbf{(O)}] \emph{(Reduced-form observability of the news channel.)} The analyst observes $D_n=\{(z_i,Y_i)\}_{i=1}^n$ i.i.d.\ with $z_i\sim\lambda$, $Y_i=A_S(z_i)+\epsilon_i$, where the noise law of $\epsilon$ (given $z$) is fixed and does not depend on the market structure $S=(Q,g)$: the observational law depends on $S$ only through its impact curve $A_S$.
\end{itemize}

The target is \emph{interventional}. For $\tau>0$ define the sensitivity intervention $I_\tau$: every agent's news sensitivity is shifted from $\eta$ to $\eta+\tau$ (a salience or mandated-disclosure scenario---precisely the class of counterfactual stress questions an agent-based simulator exists to answer). The post-intervention mean response at news magnitude $z_0>0$ is
\[
T_{z_0,\tau}(S)\;=\;\E_Q\big[g\big((\eta+\tau)z_0\big)\big] .
\]

\begin{theorem}[No reduced-form method can estimate the intervention]
\label{thm:lucas}
Assume (O) and let $g$ satisfy (G1). Fix $m\in(0,\bar\eta)$, $a\in(0,m)$ with $m+a\le\bar\eta$, and $\tau,z_0>0$. Define two market structures:
\[
S^1=\big(Q_{m,a},\,g\big),\quad Q_{m,a}=\tfrac12\delta_{m-a}+\tfrac12\delta_{m+a};
\qquad
S^2=\big(\delta_m,\,\tilde g\big),\quad \tilde g:=A_{S^1}(\cdot/m).
\]
Then:
\begin{enumerate}
\item[(i)] $\tilde g\in\cG$, and the observational laws of $D_n$ under $S^1$ and $S^2$ are identical for every $n$.
\item[(ii)] The interventional targets differ by an explicit positive amount: $T_{z_0,\tau}(S^1)-T_{z_0,\tau}(S^2)=2\Delta$ with
\[
2\Delta\;=\;\frac12\int_0^{\kappa}\Big[g'\big(u_1-\kappa+s\big)-g'\big(u_2+s\big)\Big]ds\;>\;0,
\]
where $u_1=(m-a+\tau)z_0$, $u_2=(m+a+\tau)z_0$, $\kappa=a\tau z_0/m$.
\item[(iii)] Consequently, for \emph{every} estimator $\widehat T_n$ that is a (possibly randomized) measurable function of $D_n$, and every $n$,
\[
\max_{S\in\{S^1,S^2\}}\;\E_S\big|\widehat T_n-T_{z_0,\tau}(S)\big|\;\ge\;\Delta .
\]
No consistency, no rate, and no amount of data or model capacity is possible for this target within the reduced-form class.
\end{enumerate}
\end{theorem}

\begin{proof}
(i) $Q_{m,a}$ satisfies (Q1) with $\Var(\eta)=a^2>0$, so Proposition~\ref{prop:unknowng} gives $\tilde g\in\cG$ and
$A_{S^2}(z)=\tilde g(mz)=A_{S^1}(z)$ for all $z$. Under (O) the law of $D_n$ is a functional of the impact curve alone, hence identical under the two structures, for every $n$.

(ii) Compute both targets. $T(S^1)=\tfrac12g\big((m-a+\tau)z_0\big)+\tfrac12g\big((m+a+\tau)z_0\big)=\tfrac12g(u_1)+\tfrac12g(u_2)$. For $S^2$, the intervened agent has sensitivity $m+\tau$ and response $\tilde g$, so
\[
T(S^2)=\tilde g\big((m+\tau)z_0\big)=A_{S^1}\Big(\frac{(m+\tau)z_0}{m}\Big)
=\tfrac12g\Big(\frac{(m-a)(m+\tau)z_0}{m}\Big)+\tfrac12g\Big(\frac{(m+a)(m+\tau)z_0}{m}\Big).
\]
The arguments simplify: $(m-a)(m+\tau)/m=m-a+\tau-a\tau/m$ and $(m+a)(m+\tau)/m=m+a+\tau+a\tau/m$, so
$T(S^2)=\tfrac12g(u_1-\kappa)+\tfrac12g(u_2+\kappa)$ and
\[
T(S^1)-T(S^2)=\tfrac12\big[g(u_1)-g(u_1-\kappa)\big]-\tfrac12\big[g(u_2+\kappa)-g(u_2)\big]
=\tfrac12\int_0^\kappa\big[g'(u_1-\kappa+s)-g'(u_2+s)\big]ds .
\]
All arguments are positive ($u_1-\kappa=(m-a)(m+\tau)z_0/m>0$ since $a<m$), $g'$ is strictly decreasing on $(0,\infty)$, and $u_1-\kappa+s<u_2+s$ pointwise (as $u_1<u_2$ and $\kappa>0$), so the integrand is strictly positive on $[0,\kappa]$ with $\kappa>0$; the integral is strictly positive.

(iii) Let $\widehat T_n=t_n(D_n,U)$ with $U$ exogenous randomization. By (i), $(D_n,U)$ has the same law under $S^1$ and $S^2$, so $\E_{S^2}\,h(\widehat T_n)=\E_{S^1}\,h(\widehat T_n)$ for every bounded measurable $h$, and more generally for the nonnegative integrands below. Then
\[
\E_{S^1}\big|\widehat T_n-T(S^1)\big|+\E_{S^2}\big|\widehat T_n-T(S^2)\big|
=\E_{S^1}\Big[\big|\widehat T_n-T(S^1)\big|+\big|\widehat T_n-T(S^2)\big|\Big]
\ge\big|T(S^1)-T(S^2)\big|=2\Delta
\]
by the triangle inequality, pointwise in $\widehat T_n$. The maximum of two numbers is at least their average.
\end{proof}

\medskip

\begin{corollary}[The structural estimator attains what the bound forbids]
\label{cor:lucasptmc}
Suppose the truth is $S^1$ (the PTMC premise: response function $g$ fixed by the architecture, personas drawn from an unknown $Q$ in the family \eqref{eq:family} with $\xi$ Rademacher, so that $Q_{m,a}$ corresponds to $\theta_0=(m,a^2)$ with $v_0=a^2>0$ \emph{interior}). Under (O) and the conditions (R1)--(R4) of Theorem~\ref{thm:nls} for this family, the plug-in PTMC estimator
\[
\widehat T^{\mathrm{PTMC}}_n\;=\;\E_{Q_{\hat\theta_n}}\big[g\big((\eta+\tau)z_0\big)\big]
\]
satisfies $\E_{S^1}\big|\widehat T^{\mathrm{PTMC}}_n-T_{z_0,\tau}(S^1)\big|\to0$. The bound of Theorem~\ref{thm:lucas}(iii) is not contradicted: $\widehat T^{\mathrm{PTMC}}_n$ uses the structural knowledge that the response function is $g$, information not contained in $D_n$.
\end{corollary}

\begin{proof}
$\theta_0$ is interior (both coordinates), so Theorem~\ref{thm:nls} gives $\hat\theta_n\to_p\theta_0$. The map $\theta\mapsto\E_{Q_\theta}[g((\eta+\tau)z_0)]=\tfrac12g((m+\sqrt v+\tau)z_0)+\tfrac12g((m-\sqrt v+\tau)z_0)$ is continuous, so $\widehat T^{\mathrm{PTMC}}_n\to_pT(S^1)$ by the continuous mapping theorem; boundedness ($|g|\le\sup|g|$) upgrades convergence in probability to convergence in $L^1$.
\end{proof}

\medskip

\begin{remark}[How to read the separation]
\label{rem:lucasreading}
Four points, stated carefully. (i) \emph{The theorem separates information sets, not brand names.} Any method whose output is a functional of $D_n$---a BNN posterior over a forecast function, a deep ensemble, an LLM prompted with the observational history, an LLM-agent simulator whose personas are calibrated to match $D_n$ and nothing else---sits inside the class bounded by (iii). Any method that imports the correct structural commitment from outside the data escapes it; PTMC's specific contribution is making that commitment \emph{explicit} (a fixed, documented $g$ and persona family), \emph{checkable within its class} (Theorems~\ref{thm:newsid}--\ref{thm:boundary}), and \emph{cheap to simulate} (the shared-policy architecture). (ii) \emph{The price is stated in Proposition~\ref{prop:unknowng}:} the structural commitment itself is not testable from $D_n$---in world $S^2$ the analyst assuming response $g$ would confidently recover $Q_{m,a}$ and be wrong about the intervention. Structure buys interventional identification exactly to the extent that the response family is right, which is why the companion validation protocol attacks $g$'s implications on richer observables. (iii) \emph{Richer data shrink the equivalence class but do not dissolve the point.} Under an information set larger than (O)---agent-level order flow, cross-sectional response dispersion---$S^1$ and $S^2$ may become distinguishable; the construction then migrates to whatever the enlarged observational law leaves unidentified. This is the content of the Lucas critique \citep{lucas_econometric_1976}: functionals of the observational law cannot, in general, answer questions about regimes that change the law. Theorem~\ref{thm:lucas} is that critique for market simulators, with the constant $\Delta$ computable in closed form. (iv) \emph{For purely predictive targets the theorem is silent}, deliberately: nothing here implies PTMC out-forecasts a well-calibrated reduced-form model on the next period's return, and the framework paper's comparison section should not be read otherwise. The provable advantage is confined to---and is exactly---the counterfactual questions the method was built for.
\end{remark}

\begin{remark}[When exactly the advantage holds]
\label{rem:advantageconditions}
The two separation results above license different, non-overlapping claims, and neither is unconditional. \emph{The bias-floor advantage} (Theorems~\ref{thm:quantjensen}--\ref{thm:homsep}) requires: the sensitivity distribution is truly heterogeneous ($\Var_Q(\eta)>0$, or $\Var_Q(\eta\mid\eta>0)>0$ for the uniform version of Theorem~\ref{thm:homsep}); the response function satisfies (G1); and the target is the impact-curve functional $A(z)$, compared against the class of homogeneous-population simulators $w\,g(c\,\cdot)$. Within that scope, every such rival, however well calibrated, carries bias at least $\tfrac38z^2\beta(z)\Var_Q(\eta)$ (or the uniform floor $\delta$ of Theorem~\ref{thm:homsep}); outside it---if $Q$ is in fact homogeneous---the advantage vanishes by construction, as it should. \emph{The interventional advantage} (Theorem~\ref{thm:lucas}, Corollary~\ref{cor:lucasptmc}) requires two further, jointly necessary conditions: the target must be interventional---a population-wide sensitivity shift $\eta\mapsto\eta+\tau$, not a predictive query on $D_n$---and, for PTMC specifically to attain what the bound forbids, the response family must be correctly specified, i.e.\ the true $g$ matches the architecture's fixed nonlinearity and $\theta_0$ is interior to the fitted family (the conditions of Theorem~\ref{thm:nls}). Both results rest on the same unverifiable premise: by Proposition~\ref{prop:unknowng}, whether $g$ is correctly specified cannot be checked from the impact-curve data that both separation results otherwise rely on. This is why the conditions are the content rather than a caveat appended to it: the theorems state precisely what must be true---and, by Proposition~\ref{prop:unknowng}, precisely what cannot be verified from that same observational record---for PTMC's advantage to be real rather than assumed.
\end{remark}

\section{Numerical illustration (deferred)}
\label{sec:numerics}

This paper is theory-only by design; the numerical program that illustrates it is specified in the companion experiment-design document and will be reported with the experimental study. For completeness, the four checks that theory says a correct implementation must pass are: (i) empirical coverage of the Theorem~\ref{thm:consistency} confidence intervals within Monte Carlo error of nominal, and $\hat\sigma_P^2\approx0$ under a degenerate persona distribution (negative control); (ii) measured optimal inner replication agreeing with Theorem~\ref{thm:optimalR}'s $R^*$ from pilot estimates of $(\sigma_P^2,\sigma_w^2)$; (iii) measured sensitivity slopes $|\Delta\hat\mu|/\Wone$ and $|\Delta\hat\mu|/\varepsilon_{\mathrm{BC}}$ below (typically far below) the worst-case constants of Theorem~\ref{thm:stability}; and (iv) the Jensen gap of Theorem~\ref{thm:newsid}(i) visible as the systematic separation between the impact curves of heterogeneous and slope-matched homogeneous populations, with the Theorem~\ref{thm:boundary} test rejecting homogeneity at the design's calibrated dispersion and holding size at $\sigma_\eta=0$.

\section{Discussion and open problems}
\label{sec:discussion}

\paragraph{What the theory establishes about PTMC's potential.} Four things, in increasing order of specificity to this framework. \emph{First}, PTMC is a Monte Carlo method in the full technical sense, not by analogy: unbiasedness, $N^{-1/2}$ convergence to its estimand, valid confidence intervals, and a variance decomposition with unbiased component estimators (Section~\ref{sec:estimator}) hold unconditionally, so the uncertainty-quantification machinery that makes classical Monte Carlo trustworthy transfers intact to populations of interacting learned agents. \emph{Second}, the decomposition makes compute allocation and the heterogeneity question quantitative: Theorem~\ref{thm:optimalR} turns pilot estimates of $(\sigma_P^2,\sigma_w^2)$ into an optimal redraw-versus-replicate design, and Corollary~\ref{cor:gs} turns the oldest objection to behavioral market models into a measurable quantity, functional by functional. \emph{Third}, PTMC's distance from its correctly specified counterpart is an inequality with named, measurable terms rather than an act of faith (Section~\ref{sec:stability}): error either vanishes with compute or is attributed to the calibrated distribution of trader characteristics or to the cloned policy, with the worst-case caveats of Remark~\ref{rem:budgetreading} stated rather than hidden. \emph{Fourth}, the framework's central behavioral premise is identifiable and correctly testable (Sections~\ref{sec:identification}--\ref{sec:inference}); the negative half of Proposition~\ref{prop:unknowng} is as important as the positive theorems: claims of identified heterogeneous news reaction are conditional on the response family, and empirical write-ups should say so. \emph{Fifth}, the comparison with rival method classes is now partly a theorem rather than a taxonomy (Section~\ref{sec:separation}): homogeneous-population simulators carry an explicit, irreducible bias floor whenever sensitivity is truly heterogeneous, and for interventional targets no functional of observational data---whatever its architecture---can do what the correctly specified structural estimator does, at any sample size.

\paragraph{Open problems.} Four, stated as such. (1) \emph{Mean-field limit---the convergence sense this paper does not claim.} The behavior of $\E_{\cP^{\otimes K}}[F]$ as $K\to\infty$ for a price--time-priority book with persona-conditioned policies is open; the scaling-limit literature \citep{horst_law_2017,lachapelle_efficiency_2016} treats order-flow primitives, not learned persona-conditioned policies, and nothing here depends on such a limit---the companion design probes it empirically by sweeping $K$. (2) \emph{Ergodicity of the market chain.} Verifying (A4)---or any quantitative mixing condition with a non-vacuous constant---for a realistic book is open (Remark~\ref{rem:doeblin}); progress here would convert Theorem~\ref{thm:uniform} from conditional to unconditional. (3) \emph{Path functionals uniform in $T$.} Extending horizon-uniform stability to path suprema (drawdowns) plausibly requires regeneration-based arguments rather than marginal coupling; we do not know the right statement. (4) \emph{Nonparametric estimation of $Q$.} Theorem~\ref{thm:newsid}(ii) inverts a moment map that is severely ill-posed as an estimation problem (compactly supported mixing distributions recovered from smooth mixtures typically admit only logarithmic rates, as in the deconvolution and random-coefficients literatures \citep{beran_estimating_1992,gautier_nonparametric_2013}); quantifying the optimal rate for the impact-curve inversion, and whether the novelty dimension of the empirical design improves it, is open. The parametric Theorems~\ref{thm:nls}--\ref{thm:boundary} are the pragmatic instrument in the meantime.

\section{Conclusion}
\label{sec:conclusion}

This paper's aim was to make ``reliable'' a set of theorems rather than an adjective, and thereby to demonstrate what Persona-Trained Monte Carlo can already claim---and exactly where those claims stop. Four conclusions.

\emph{PTMC inherits the guarantees of classical Monte Carlo.} For a fixed market configuration, the estimator is unbiased, converges at the $N^{-1/2}$ rate to its estimand, admits valid confidence intervals, and its variance splits into persona-draw and within-run components with unbiased estimators of each (Theorem~\ref{thm:consistency})---unconditionally, because the outer persona loop preserves the i.i.d.\ structure classical Monte Carlo theory needs. The practical consequences are immediate: a pilot run prices the two variance components, Theorem~\ref{thm:optimalR} converts them into the optimal redraw-versus-replicate design, and the permutation test of Remark~\ref{rem:permutation} decides, functional by functional, whether learned heterogeneity is doing any work. The Gode--Sunder objection stops being a debate and becomes a number.

\emph{PTMC's wrongness is budgetable.} Corollary~\ref{cor:budget} splits the distance to the correctly specified simulator into Monte Carlo error, policy-cloning error, and error in the calibrated distribution of trader characteristics---each separately measurable, the first driven to zero by compute, the other two only by better data. The bound is conditional on a Lipschitz property that Lemma~\ref{lem:softmax} verifies for the actual softmax architecture, and it is a worst case whose constants price every bot's every decision as pivotal (Remark~\ref{rem:budgetreading}); under an ergodicity assumption its horizon dependence disappears for terminal and time-averaged functionals (Theorem~\ref{thm:uniform}). Two gaps are stated rather than obscured: the $K\to\infty$ limit is posed as an open mathematical problem, and no theorem certifies the simulator class against the real market---the latter is the empirical task of the companion validation protocol, not something a stability bound can supply.

\emph{PTMC's behavioral core is falsifiable, and the test is now correct.} The aggregate news-impact curve strictly separates heterogeneous from homogeneous news sensitivity, identifies the sensitivity distribution from its behavior near the origin, and still identifies the shape of the reacting subpopulation when the curve is known only up to scale (Theorem~\ref{thm:newsid}, Corollary~\ref{cor:scale}); the $\sqrt n$ estimator and the boundary-corrected $\tfrac12\chi^2_0+\tfrac12\chi^2_1$ test (Theorems~\ref{thm:nls}--\ref{thm:boundary}) make the hypothesis operational, replacing a resampling procedure that would have been invalid at exactly the null of interest \citep{andrews_inconsistency_2000}. Proposition~\ref{prop:unknowng} fixes the epistemic boundary: heterogeneity is identified \emph{given} the architecture's response nonlinearity, and honest empirical claims must carry that conditioning.

\emph{PTMC's advantage over rival methods is conditional, and now quantified precisely rather than loosely.} Two advantages, two disjoint sets of conditions, stated together in Remark~\ref{rem:advantageconditions}. The first needs only that the sensitivity distribution be truly heterogeneous ($\Var_Q(\eta)>0$) and the target be the impact-curve functional: under exactly that condition, every homogeneous-population simulator---however calibrated---carries the bias floor of Theorems~\ref{thm:quantjensen}--\ref{thm:homsep}, while the structural estimator converges; if $Q$ is in fact homogeneous, the advantage is vacuous by construction, as it should be. The second needs the target to be interventional---a population-wide sensitivity shift, not a predictive query---\emph{and} the response family to be correctly specified: under those two conditions jointly, Theorem~\ref{thm:lucas} shows that no method operating on observational data alone---Bayesian neural networks, ensembles, or LLM-based forecasters---can beat an explicit worst-case constant at any sample size, and Corollary~\ref{cor:lucasptmc} shows the correctly specified structural estimator attains what that bound forbids. The conditions are the content, and one of them is load-bearing in a way the other is not: correct specification of the response family is itself untestable from the same observational data (Proposition~\ref{prop:unknowng}), so the interventional advantage is real only to the extent that this one structural commitment---explicit, checkable in principle from richer data, and exactly what the companion protocol is designed to probe---actually holds. For purely predictive targets no dominance is claimed, because none is provable.

The demonstrated potential, then, is this: among agent-based market models, PTMC can be operated as a statistical instrument---with quantified sampling error, an audited input-error budget, principled compute allocation, and a pre-registered, correctly calibrated test of its own central premise. The framework paper proposes the estimator and its validation protocol; the experiment-design document pre-registers the studies; the theorems above specify exactly which numbers that program must report---interval coverage, $(\hat\sigma_P^2,\hat\sigma_w^2,R^*)$, sensitivity slopes measured against the worst-case constants, and the impact-curve test statistics---for the word ``reliable'' to be earned rather than asserted. For a reader arriving from the causal-machine-learning literature, the trade this paper makes explicit is worth restating plainly: predictive capacity is bought from data through $\pi_\phi$, interventional capacity is bought from an explicit, checkable structural commitment on $\cP$ and $g$, and Theorem~\ref{thm:lucas} is the proof that no amount of the former substitutes for the latter.

\appendix

\section{Verification of assumption (A1) for softmax policies}
\label{app:softmax}

\begin{lemma}[Softmax policies are TV-Lipschitz in the persona]
\label{lem:softmax}
\leavevmode
\begin{enumerate}
\item Let $\cA$ be finite and let $\pi(\cdot\mid s,p)=\mathrm{softmax}\big(f(s,p)\big)$, where the logit map $f(s,\cdot):P\to\R^{|\cA|}$ satisfies $\|f(s,p)-f(s,p')\|_\infty\le L_0\|p-p'\|$ for all $s,p,p'$. Then (A1) holds with $L=(1-1/|\cA|)\,L_0 \le L_0$.
\item (Composite actions.) Let the action be $a=(a^1,a^2)$, where $a^1$ is drawn from a head as in (1) with constant $L_1$, and, given $a^1$, the continuous mark $a^2$ is drawn from a conditional distribution $\kappa(\cdot\mid a^1,s,p)$ satisfying $\sup_{a^1,s}\TV\big(\kappa(\cdot\mid a^1,s,p),\kappa(\cdot\mid a^1,s,p')\big)\le L_2\|p-p'\|$. Then the joint action distribution satisfies (A1) with $L=(1-1/|\cA|)L_1+L_2$.
\item (Gaussian mark head.) If $a^2\sim\mathcal N\big(m(a^1,s,p),\sigma^2\big)$ with fixed $\sigma>0$ and $|m(a^1,s,p)-m(a^1,s,p')|\le L_m\|p-p'\|$, then the condition of (2) holds with $L_2=L_m/(\sigma\sqrt{2\pi})$.
\end{enumerate}
\end{lemma}

\begin{proof}
(1) Write $u=f(s,p)$, $v=f(s,p')$, and $\pi_a(w)=e^{w_a}/\sum_b e^{w_b}$ for $w\in\R^{|\cA|}$. The gradient of $\pi_a$ is $\nabla_w\pi_a(w)=\pi_a(w)\,(e_a-\pi(w))$, whose $\ell_1$ norm is
\[
\|\nabla_w\pi_a(w)\|_1=\pi_a(w)\Big[(1-\pi_a(w))+\sum_{b\ne a}\pi_b(w)\Big]=2\pi_a(w)\big(1-\pi_a(w)\big).
\]
Along the segment $w(t)=v+t(u-v)$, $t\in[0,1]$, the fundamental theorem of calculus and H\"older's inequality ($|\nabla\pi_a\cdot h|\le\|\nabla\pi_a\|_1\|h\|_\infty$) give
\[
\sum_{a\in\cA}\big|\pi_a(u)-\pi_a(v)\big|
\;\le\;\int_0^1\sum_a\big\|\nabla_w\pi_a(w(t))\big\|_1\,dt\;\cdot\;\|u-v\|_\infty .
\]
For any probability vector $q$, $\sum_a2q_a(1-q_a)=2\big(1-\sum_aq_a^2\big)\le2\big(1-1/|\cA|\big)$, by Cauchy--Schwarz ($\sum q_a^2\ge1/|\cA|$). Since for finite spaces $\TV=\frac12\sum_a|\cdot|$, we get
$\TV(\pi(u),\pi(v))\le(1-1/|\cA|)\|u-v\|_\infty\le(1-1/|\cA|)L_0\|p-p'\|$.

(2) Let $\lambda,\lambda'$ be the joint laws of $(a^1,a^2)$ under $p,p'$. For any measurable $C$,
\[
\lambda(C)-\lambda'(C)=\underbrace{\int\kappa(C_{a^1}\mid a^1,s,p)\,\big(\pi-\pi'\big)(da^1)}_{(\mathrm I)}
+\underbrace{\int\big[\kappa(C_{a^1}\mid a^1,s,p)-\kappa(C_{a^1}\mid a^1,s,p')\big]\pi'(da^1)}_{(\mathrm{II})},
\]
where $C_{a^1}$ is the section of $C$. The integrand of (I) lies in $[0,1]$, so $|(\mathrm I)|\le\TV(\pi,\pi')$ (for any $[0,1]$-valued $h$, $|\int h\,d(\pi-\pi')|\le\TV(\pi,\pi')$, as is seen by writing $\pi - \pi'$ in terms of its positive and negative parts, each of mass $\TV(\pi,\pi')$); and $|(\mathrm{II})|\le\sup_{a^1}\TV(\kappa(\cdot\mid a^1,s,p),\kappa(\cdot\mid a^1,s,p'))$. Taking suprema over $C$ and applying (1) and the hypothesis gives the claim.

(3) For $\delta=|m-m'|$, the total variation distance between $\mathcal N(m,\sigma^2)$ and $\mathcal N(m',\sigma^2)$ equals $2\Phi\big(\delta/(2\sigma)\big)-1$: the density difference changes sign exactly at the midpoint $(m+m')/2$, so the supremum over sets is attained at the half-line where the first density dominates, giving $\Phi(\delta/2\sigma)-\Phi(-\delta/2\sigma)$. Since $\Phi(x)-\tfrac12\le\varphi(0)\,x=x/\sqrt{2\pi}$ for $x\ge0$ (the standard normal density $\varphi$ is maximal at $0$), this is at most $\delta/(\sigma\sqrt{2\pi})$.
\end{proof}

\medskip

\begin{remark}[Persona-dependent scale]
\label{rem:scale}
If the mark head's scale also depends on the persona, $a^2\sim\mathcal N(m_p,\sigma_p^2)$, a fully explicit bound is still available via Pinsker's inequality $\TV\le\sqrt{\mathrm{KL}/2}$ and the exact Gaussian divergence
$\mathrm{KL}\big(\mathcal N(m_1,\sigma_1^2)\,\|\,\mathcal N(m_2,\sigma_2^2)\big)=\tfrac12\big[t-\log(1+t)\big]+\frac{(m_1-m_2)^2}{2\sigma_2^2}$, $t:=\sigma_1^2/\sigma_2^2-1$.
On $|t|\le\tfrac12$ one has $t-\log(1+t)\le t^2$ (the function $h(t) = t^2-t+\log(1+t)$ has $h(0) = 0$ and $h'(t)=t(2t+1)/(1+t)$, which is $\ge 0$ on $[0,\tfrac12]$ and $\le 0$ on $[-\tfrac12,0]$), whence
$\TV\le\frac{|t|}{2}+\frac{|m_1-m_2|}{2\sigma_2}$, using $\sqrt{x+y}\le\sqrt x+\sqrt y$. If $p\mapsto(m_p,\sigma_p^2)$ is Lipschitz and $\sigma_p^2$ is bounded below, (A1) follows with an explicit constant. Log-normal size heads reduce to the Gaussian case because total variation is invariant under the (common, bijective) exponential reparameterization of both measures.
\end{remark}

\noindent\textit{AI-assisted writing: A large language model assisted with drafting and editing; the author takes full responsibility. AI is not listed as an author.}

\bibliographystyle{plainnat}
\bibliography{ptmc_theory}

\end{document}